\documentclass{article}

\usepackage{cite}
\usepackage{float}
\usepackage{amsthm, amsmath,amssymb,amsfonts}
\usepackage{algorithmic}
\usepackage{multirow}
\usepackage{graphicx}
\usepackage{textcomp}
\usepackage{makecell}
\usepackage{courier}
 \usepackage{hyperref}
 \usepackage{longtable}
 \usepackage{xcolor}
 \usepackage{soul}
 \usepackage[margin=1in]{geometry}
\usepackage{xcolor}  
\sethlcolor{yellow}  
\newcommand{\comb}[1]{\textcolor{black}{#1}}

\newcommand{\comblue}[1]{\textcolor{black}{#1}}

\newtheorem{proposition}{Proposition}
\newtheorem{lemma}{Lemma}
\newtheorem{theorem}{Theorem}
\newtheorem{example}{Example}

\begin{document}
\begin{center}

{\bf Designing ReLU Generative Networks to Enumerate Trees with a Given Tree Edit Distance\\}

\vspace{1cm}
Mamoona Ghafoor
~and
 Tatsuya Akutsu


  {Bioinformatics Center, Institute for Chemical Research, Department of Intelligence Sciences and Technology, 
Kyoto University, Uji 611-0011, Japan }

{Email: mamoona.ghafoor@kuicr.kyoto-u.ac.jp; takutsu@kuicr.kyoto-u.ac.jp
}
\end{center}
\vspace{1cm}

\begin{abstract}
The generation of trees with a specified tree edit distance has significant applications across various fields, including computational biology, structured data analysis, and image processing.
Recently, generative networks have been increasingly employed to synthesize new data that closely resembles the original datasets.
However, the appropriate size and depth of generative networks required to generate data with a specified tree edit distance remain unclear.
In this paper, we theoretically establish the existence and construction of generative networks capable of producing trees similar to a given tree with respect to the tree edit distance.
Specifically, for a given rooted, ordered, and vertex-labeled tree $T$ of size $n+1$ with labels from an alphabet $\Sigma$, and a non-negative integer $d$, we prove that any rooted, ordered, and vertex-labeled tree over $\Sigma$ with tree edit distance at most $d$ from $T$ can be generated using an appropriate random input sequence to a ReLU-based generative network of size $\mathcal{O}(n^3)$ and constant depth. 
The proposed networks were implemented and evaluated for generating trees with up to 21 nodes.
Due to their deterministic architecture, the networks successfully generated all valid trees within the specified tree edit distance. In contrast, state-of-the-art graph generative models GraphRNN and GraphGDP, which rely on non-deterministic mechanisms, produced significantly fewer valid trees, achieving validation rates of only up to 35\% and 48\%, respectively.
These findings provide a theoretical foundation towards construction of compact generative models and open new directions for exact and valid tree-structured data generation.
An implementation of the proposed networks is available at 
\url{https://github.com/MGANN-KU/TreeGen\_ReLUNetworks}.
\end{abstract}
\noindent
{\bf Keywords:}{
{Generative networks; ReLU function; trees; tree edit distance; enumeration; Euler string}


\section{Introduction}
Over the past few years, generative networks have been widely studied due to its vast applications in different fields such as natural language
processing, data augmentation, DNA sequence synthesis, and
drug discovery~\cite{KN2017, TB2013, BA2017, GE2022}.
Generative networks are a class of machine learning models that learn the underlying patterns, structures, and dependencies in training data. By capturing this statistical information, they can generate new data samples that resemble the original data. These models are not limited to data synthesis but are also used in tasks such as data augmentation, imputation, and representation learning. Applications include generating realistic images, text, audio, and other complex data modalities~\cite{WC2019, VD2016, AN2022}.
For example, the application of generative models in bioinformatics includes motif discovery, secondary structure prediction, drug discovery, cancer research, the generation of new molecules and the analysis of single-cell RNA sequencing data~\cite{ZJ2014}, \cite{LJ2018}, \cite{SK2020}, \cite{GC2020}. 

There are various types of generative models, each with distinct characteristics. Autoencoder-based models include variational autoencoders~(VAEs)~\cite{KD2019} and denoising autoencoders~(DAEs)~\cite{BY2013}, which are designed to learn compact representations of data by encoding and decoding it. 
Generative adversarial networks~(GANs)~\cite{GI2020}, including specialized versions like deep convolutional generative adversarial networks~(DCGANs)~\cite{GF2018}, use adversarial training to generate realistic data by having a generator and discriminator compete against each other. 
Deep belief networks, such as deep Boltzmann machines~(DBMs)~\cite{HG2009}, are probabilistic models that represent complex data distributions through a stack of restricted Boltzmann machines. 
Generative stochastic networks~(GSNs)~\cite{AG2016} use stochastic processes to generate data by iteratively refining it. 
Autoregressive models, including pixel convolutional neural networks~(PixelCNN)~\cite{VO2016} and pixel recurrent neural networks~(PixelRNN)~\cite{VD2016b}, model the distribution of image pixels, generating data pixel by pixel in a sequential manner. The deep recurrent attentive writer (DRAW) model~\cite{GK2015} combines recurrent neural networks and attention mechanisms to generate images, focusing on specific parts of the data during the generation process. 
Diffusion models have recently gained significant popularity in generative AI.
These models generate data by first learning how to gradually add noise to real data until it becomes random noise. Then, they are trained to reverse this process, step by step, by removing the noise and recovering the original data distribution. During generation, the model starts with pure noise and progressively denoises it to produce a realistic sample, like an image or audio clip. 
This step-by-step denoising is guided by a neural network, often trained to predict the noise added at each stage~\cite{HJ2020}, \cite{YL2023}.

Selecting the right function family and network model is essential in machine learning. A function family that is too broad may cause high computational costs and overfitting, while a limited one might not produce accurate predictions~\cite{KA2022}. 
Choosing the appropriate network remains challenging, as there is no clear choice for every problem. 
The universal approximation theorem states that a two-layer neural network can approximate any Borel measurable function~\cite{HSW1989}. 
But such networks often require a large number of nodes. 
Studies have examined how the choice of function family relates to network size, revealing that deeper networks significantly enhance representational power~\cite{MPCB2014, RPKGS2017}. 
Not all functions can be efficiently represented by any architecture. 
For instance, Telgarsky~\cite{T2015} identified functions requiring exponentially more nodes in shallow networks compared to deep ones. 
Szymanski and McCane~\cite{SM2014} showed that deep networks are well-suited for modeling periodic functions, while Chatziafratis et al.~\cite{CNPW2019} established width lower bounds based on depth for such functions. 
Hanin and Rolnick~\cite{BR2019} further found that networks with piecewise linear activation functions do not exponentially increase their expressive regions. 
Additionally, Bengio et al.\cite{YDS2010} and Biau et al.\cite{BSW2019} demonstrated that decision trees and random forests can be approximated by neural networks using sigmoidal, Heaviside, or tanh activation functions. 
Kumano and Akutsu~\cite{KA2022} later extended this result to networks using ReLU and similar activations.
Recently, Ghafoor and Akutsu~\cite{MT2024} discussed the existence of 
generative networks with ReLU as activation function and constant depth to generate similar strings with a given edit distance.  

Selkow~\cite{SM1977} introduced the problem of tree edit distance as a generalization of the classical string edit distance problem. 
The tree edit distance problem has several applications in applied fields including  computational biology~\cite{Gus97, SZ90, HTGK03, Wat95}, analysis of structured data~\cite{BGK03, Cha99, FLMM09}, and image processing~\cite{BK99, KTSK00, KSK01, SKK04}. 
Different algorithms have been developed to compute the tree edit distance between two rooted, labeled and ordered trees. 
For instance, Tai~\cite{KC1979} proposed an algorithm for the tree edit problem with time complexity $\mathcal{O}(n^6)$, where $n$ is the size of the underlying tree. Zhang and Shasha~\cite{KZ1989}, Klein~\cite{PN1998} and Demaine et al.~\cite{ED2010} proposed improved algorithms with a time complexity $\mathcal{O}(n^4)$, $\mathcal{O}({n^3}\log n)$ and $O(n^3)$, respectively. 
Later on Bringmann et al.~\cite{KB2020} further improved the complexity to 
$\mathcal{O}(n^{3-\epsilon})$ for weighted trees. 
In 2022, Mao~\cite{XM2022} introduced an algorithm for unweighted trees with complexity $\mathcal{O}(n^{2.9148})$. 
Recently, Nogler et al.~\cite{NJ2024}, proposed an efficient algorithm with  complexity $\mathcal{O}(n^{ 3/2 \Omega( \sqrt{\log n})})$ for weighted trees 
and $\mathcal{O}(n^{2.6857})$ for unweighted trees.  

Recent years have witnessed significant advancements in structured generative networks, with a growing focus on models that balance empirical performance and structural validity. GraphRNN is a foundational autoregressive model that generates graphs by sequentially adding nodes and edges, widely used in molecular design \cite{You2018GraphRNN}. Building on this, Wang et al.~\cite{Wang2025LearningOrder} proposed a variational autoregressive model that learns generation order dynamically, achieving state-of-the-art molecular graph results without diffusion. AutoGraph \cite{Chen2025AutoGraph} applies transformers to autoregressively generate graphs as sequences. TreeGAN by Liu et al.~\cite{Liu2018TreeGAN} is a syntax-aware generative adversarial network designed for sequence generation that respects tree-structured syntactic constraints. In parallel, diffusion-based generative models have recently advanced structured data generation. For instance, Huang et al.~\cite{Huang2022GraphGDP} proposed a continuous-time generative diffusion model, GraphGDP, to generate permutation-invariant graphs. 
Liu et al.~\cite{Liu2025BetaDiff} introduced a beta-noise process that effectively models both discrete graph structures and continuous node attributes, achieving strong results on biochemical and social network benchmarks. The framework by Madeira et al.~\cite{Madeira2024StructGraph} enforced hard structural constraints such as planarity or acyclicity via an edge-absorbing noise mechanism, ensuring generated graphs rigorously maintain desired properties throughout the diffusion process. 
While existing generative models offer powerful empirical frameworks for producing high-quality and diverse structured data, their guarantees are inherently probabilistic and data-dependent. These methods require training on limited datasets, and their performance is influenced by the quality and coverage of this data. As a result, they cannot provide exact enumeration of the underlying combinatorial space, nor can they ensure complete validity or coverage of all possible structured instances.

As a step towards addressing these issues, we study the exact generation of rooted, ordered, and vertex-labeled trees similar to a given tree using neural networks, and prove the existence of ReLU-based generative networks for the tree edit distance problem.
Given a rooted, ordered, and vertex-labeled tree $T$ of size $n+1$ with labels from a symbol set $\Sigma$, we establish the existence of a ReLU-activated network that takes as input a finite random sequence $x$, and outputs a similar tree with edit distance at most $d$ from $T$. 
The network is deterministic in the sense that, for each random input sequence $x$, it systematically identifies and performs the specified substitution, deletion, and insertion operations to produce a unique output tree $T'$ whose tree edit distance from $T$ is at most $d$ (see the proofs of Theorems~\ref{thm:Esnn}, \ref{thm:Ednn}, \ref{thm:EInn}, and \ref{thm:Enn} for details).   
Different random input sequences may lead to the same output tree. 
Nevertheless, by considering an appropriate collection of random input sequences that covers all admissible edit operations, every tree similar to $T$ within distance $d$ can be generated.
The key idea of our approach is to first construct a directed, rooted, ordered, and edge-labeled tree based on $T$.
We then reduce the tree edit distance problem to a string edit distance problem by representing the tree as an Euler string~\cite{PN1998}, obtained through a depth-first search (DFS) traversal.
The proposed networks are applied on trees with up to 21 nodes to generate similar trees, and are compared with the state-of-the-art graph generative models GraphRNN by You et al.~\cite{You2018GraphRNN} and GraphGDP by Huang et al.~\cite{Huang2022GraphGDP}.
An implementation of the proposed networks is available at~\url{https://github.com/MGANN-KU/TreeGen\_ReLUNetworks}.

The paper is organized as follows: 
Preliminaries are discussed in Section~\ref{sec:pre}. 
ReLU generative networks that can identify the indices and labels of the directed edges in the Euler string are discussed in Section~\ref{sec:OutE}. 
Existence of  
ReLU networks to generate any tree with tree edit distance at most $d$ due to substitution operations is discussed in Section~\ref{sec:ESGR}. 
Existence of ReLU networks to generate any tree with edit distance at most (resp., exactly) $d$ due to deletion (resp., insertion) operations is discussed in 
Section~\ref{sec:EdGR}. 
Generation of any tree with tree edit distance at most $d$ due 
to simultaneous application of deletion, substitution and insertion operations by using  a ReLU network is discussed in Section~\ref{sec:EIGR}. 
Computational experiments are discussed in Section~\ref{sec:comp}.
A conclusion and future directions are given in Section~\ref{sec:concl}. 
Proofs of some theorems, examples and explanations of the program codes with sample instances are 
given in Appendix~\ref{sec:app}.   

\section{Preliminaries}\label{sec:pre}
Edit distance between two vertex-labeled, rooted and ordered trees $T$  and $U$  is defined as the minimum number of operations needed to transform $T$  into $U$. 
These operations are substitution, deletion, and insertion. 
Substitution involves simply changing the label of a node in $T$; 
deletion removes a non-root node $a$ in $T$, reassigning its parent $b$ as the new parent of all children of $a$; and 
insertion is the complement of the deletion operation, i.e., 
a node $a$ is inserted as a child of a node $b$, and $a$ is set as the new parent of an ordered subset of consecutive children of $b$. 
The order among the children is preserved during both the deletion and insertion operations, as illustrated in Fig.~\ref{fig:ETD}.
\begin{figure}[h]
	\centering
	\includegraphics[scale = 0.6]{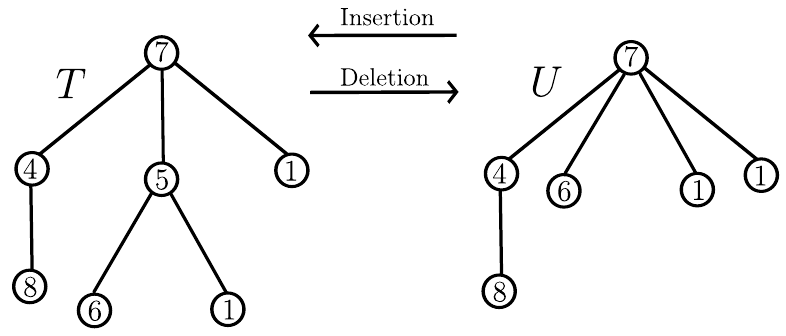}
	\caption{Tree deletion and insertion operations. 
	In $T$, the node with label 7 is the parent of the node with label 5, which has been deleted. The node with label 7 becomes the new parent of the children with labels 1 and 6 of node 5 in $U$. Similarly, a node with label 5 is inserted as a child of the node with label 7 in $U$, and the nodes with labels 1 and 6 are set as children of node 5 in $T$.
The order among the children 1 and 6 is preserved in the deletion and insertion operations. 
}\label{fig:ETD}
\end{figure}

Let $T$ be a vertex-labeled, rooted and ordered tree with $n$ edges ($n+1$ vertices) with
vertex labels from  the set $\Sigma = \{1, 2, \ldots, m\}$, and
 left-to-right ordering on the siblings of each vertex. 
 We consider the depth-first search (DFS) index on the vertices of $T$ starting from the root with index $0$. 
For $T$, we define a directed edge-labeled, rooted and ordered tree with $n+1$ vertices and $2n$ edges as follows: 
replace the edge between any two adjacent vertices $u$ and $v$ with labels $a$ and $b$, resp., where $u$ is the parent of $v$, 
by two directed edges $(u, v)$ and $(v, u)$ with labels $b$ and $b+m$, respectively.
In this setting, we call $(u, v)$ and $(v, u)$, the inward edge and the outward edge, resp.,  of the vertex $v$. 
If $i$ is the DFS index of $v$, then we call the inward and outward edges of $v$, the inward and outward edges of $i$.
We call $(v, u)$, the outward edge of the inward edge $(u, v)$. 
 The Euler string of $T$ is defined to be the string obtained by 
 listing the labels of the inward and outward edges of the directed tree corresponding to $T$ in the DFS order on edges starting from index $1$. 
 We denote by $E(T)=t_1, t_2, \ldots, t_{2n}$, the Euler string of $T$ with $n$ edges.
 An example tree $T$, its directed tree, and Euler string are given in Figs.~\ref{fig:ET}(a), (b), and (c), respectively. 
 Henceforth, we will use the terms, edge and label interchangeably. 
\begin{figure}[t!]
	\centering
	\includegraphics[scale = 0.6]{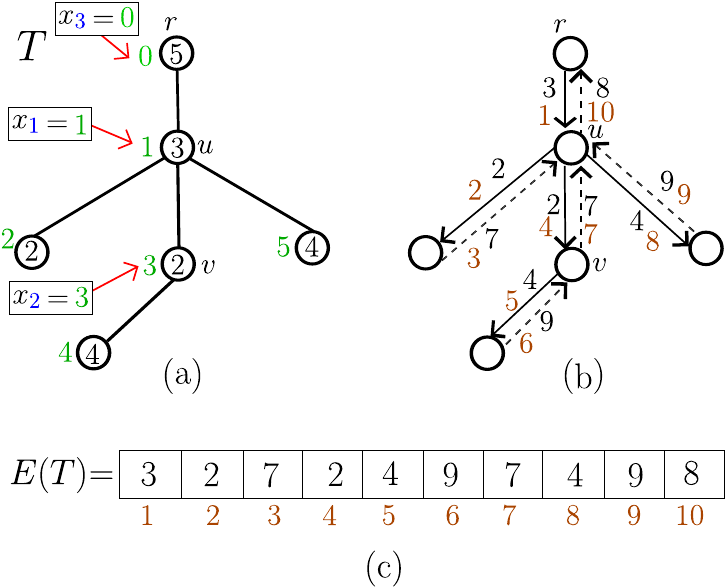}
	\caption{(a) A vertex-labeled, rooted and ordered tree $T$ with six vertices, root $r$, label set $\Sigma = \{1, 2,3, 4, 5\}$ and a random sequence $x_1, x_2, x_3 = 
	1, 3, 0$, where the labels are depicted inside the vertices, and the DFS indices are shown in green;
	(b) The directed tree corresponding to $T$ given in (a). The inward edges and outward edges are depicted by solid and dashed directed lines, respectively. The labels and DFS indices of these edges are shown in black and brown color, respectively. 
	The vertex $u$ with label $3$ is the parent of $v$ with label $2$. 
	Corresponding to the edge $uv$ in $T$, there is an inward edge 
	$(u, v)$ and an outward edge $(v, u)$ with labels $2$ and $7$, respectively, in the directed tree. 
	These edges $(u, v)$ and $(v, u)$ are the inward and outward, resp., edges of 
	$x_2 = 3$ since the DFS index of $v$ is 3 in $T$. Note that there is no edge that corresponds to $x_3=0$;  and 
	(c) The Euler string $E(T)$.
}\label{fig:ET}
\end{figure}

Observe that a tree can be completely determined by its Euler string, i.e., 
 $E(T)$ is a canonical representation of $T$ when the labels of roots of the underlying trees is fixed. 
Therefore any tree edit operations (substitution, deletion, and insertion) on a 
non-root vertex $u$ of a given tree can be viewed as edit operations on the Euler string of the tree on the entries that correspond to the inward edge $(u, v)$ and the outward edge $(v, u)$ with some refinements to obtain the desired tree. 
Furthermore, a vertex in $T$ can be specified by its DFS index. 
Therefore in the rest of the paper, we will use a random sequence $x_1, x_2, \ldots, x_d$, 
$d \geq 1$, with integers $x_j \in [0, n]$, unless stated otherwise, to specify the DFS indices of the vertices under consideration in a tree with $n+1$ vertices. 
In this work, we focus on generation of similar trees with the same root as that of the input tree, therefore $x_j = 0$ is ignored in substitution and deletion operations. Similarly, repeated entries in the case of substitution and deletion operations are ignored. 
An example random sequence is given in Fig.~\ref{fig:ET}(a). 
In the rest of the discussion, we call a vertex-labeled or edge-labeled, rooted and ordered tree simply a tree. 
\comblue{A list of symbols, variables, and their descriptions used throughout the discussion is provided in Table~\ref{tab:reset}.}

\begin{table}[h!]
\centering
\caption{List of symbols, variables, and their descriptions.}
\comblue{\begin{tabular}{|lp{14cm}|}
\hline
\textbf{Symbols} & \textbf{Explanations} \\
\hline
$T$ & Vertex-labeled, rooted and ordered tree. See Fig.~\ref{fig:ETD}.\\
$(u, v)$ & Inward edge of $v$, where $u$ is the parent of $v$. See Fig.~\ref{fig:ET}(b.)\\
$(v, u)$ & Outward edge of $v$,  where $u$ is the parent of $v$. See Fig.~\ref{fig:ET}(b).\\
$E(T)$ & Euler string of tree $T$. See Fig.~\ref{fig:ET}(c). \\
$n$ &  Number of edges in a tree.   \\
$d$ & Edit distance.   \\
$B$ &  A sufficiently large number.  \\
$\Sigma$ & Set $\{1, 2, \ldots, m\}$ of labels. \\
$\rm{in}_{\ell i}$ & Number of inward edges between ${\ell}$-th and $i$-th entries of an Euler string. See Fig.~\ref{fig:inofout} . \\
$\rm{out}_{\ell i}$ &  Number of outward edges between ${\ell}$-th and $i$-th entries of an Euler string. See Fig.~\ref{fig:inofout}. \\
DFS index &  Depth first search indexing of vertices. See Fig.~\ref{fig:ET}(a)   \\
TS$_d$& Generative ReLU network for tree edit distance $d$ due to substitution. See Fig.~\ref{fig:sub_N}.  \\
TD$_d$ &Generative ReLU network for tree edit distance $d$ due to deletion. See Fig.~\ref{fig:Del_N}.  \\
TI$_d$& Generative ReLU network for tree edit distance $d$ due to insertion. See Fig.~\ref{fig:in_N}.  \\
TE$_d$& Generative ReLU network for tree edit distance $d$ due to substitution, deletion or insertion. See Fig.~\ref{fig:Uni_N}. \\
$x_1, x_2, \ldots, x_{2d}$ & A random input sequence for TS$_d$. $x_j, 1 \leq j \leq d$
specifies the DFS index of a vertex for substitution, and $x_{d+j},  1 \leq j \leq d$ denote the value to be substituted. See Fig.~\ref{fig:sub_N}. \\
$x_1, x_2, \ldots, x_d$ & A random input sequence for TD$_d$. 
$x_j$ specifies the DFS index of a vertex of a tree to be deleted. See Fig.~\ref{fig:ET}(a). \\
$x_1, x_2, \ldots, x_{4d}$ & 
A random input sequence for TI$_d$. $x_j, 1 \leq j \leq d$
specifies the DFS index of a vertex for insertion, and $x_{d+j}$ and $x_{2d+j},  1 \leq j \leq d$ specify bounds on the children of the vertex with index $x_j$, and  $x_{3d+j},  1 \leq j \leq d$ denotes the value to be inserted. See Fig.~\ref{fig:in_N}.  \\
$x_1, x_2, \ldots, x_{7d}$ & 
A random input sequence for TE$_d$. $x_j, 1 \leq j \leq d$
specifies the input for deletion, 
$x_{d+j}, 1 \leq j \leq 2d$ specifies the input for substitution, and 
$x_{3d+j}, 1 \leq j \leq 4d$ specifies the input for insertion. See Fig.~\ref{fig:Uni_N}.  \\
&\\
&{ All local variables used in Lemma~\ref{thm:inward} are explained in Example~\ref{ex:in} and Fig.~\ref{fig:inward}.}\\
& {All local variables used in Proposition~\ref{pro:e-bar1} are explained in Example~\ref{ex:desc} and Fig.~\ref{fig:inofout}.}\\
&{ All local variables used in Lemma~\ref{thm:ebar4} are explained in Example~\ref{ex:out} and Fig.~\ref{fig:outward}.}\\
&{ All local variables used in Theorem~\ref{thm:Esnn} are explained in Example~\ref{exa:Sub} and Fig.~\ref{fig:sub}.}\\
&{ All local variables used in Theorem~\ref{thm:Ednn} are explained in Example~\ref{exa:Del} and Fig.~\ref{fig:Del}.}\\
&{ All local variables used in Theorem~\ref{thm:EInn} are explained in Example~\ref{exa:ins} and Fig.~\ref{fig:Ins}.}\\
&{ All local variables used in Theorem~\ref{thm:Enn} are explained in Example~\ref{exa:Uni}.}\\
\hline
\end{tabular}}
\label{tab:reset}
\end{table}

\section{Identification of Edge Labels by ReLU Network}\label{sec:OutE}
We focus on designing generative networks with ReLU as an activation function to generate any tree that are similar to a given tree. More precisely, we are interested in the following problem:\smallskip\\
{\bf Input:} A rooted, ordered, and vertex-labeled tree $T$ with labels from an alphabet $\Sigma$, and a non-negative integer $d$.\\
{\bf Output:} Construct generative networks with ReLU as an activation function that can generate any rooted, ordered, and vertex-labeled tree over $\Sigma$ with tree edit distance at most $d$ from $T$. 

We target this problem by reducing the tree edit distance problem to the string edit distance problem by representing trees as their Euler strings as explained in Section~\ref{sec:pre}.
As a sub-task, the positions and labels of the under consideration inward and outward edges in the Euler string are required to perform the edit operations. 
However such positions and labels are not readily available. 
Therefore we first discuss the existence of ReLU networks to identify the positions and labels of the 
inward and outward edges in an Euler string corresponding to a given random sequence 
$x_1, x_2, \ldots, x_d$ in \comblue{Lemma~\ref{thm:inward}}. 

\begin{lemma}
\label{thm:inward}
\comblue{Let $T$ be a tree with $n$ edges, 
and $ x = x_1, x_2, \ldots, x_{d}$ be a random DFS sequence of integers 
over the interval $[0, n]$.  
Then there exists a ReLU network with size 
$\mathcal{O}(dn)$ and constant depth that can identify the label
of inward edge of the vertex with non-zero DFS index $x_j$ in the Euler string $E(T)$.} 
\end{lemma}
\begin{proof}
Let $E(T)=t_1, t_2, \ldots, t_{2n}$.  
The following system of equations  can be used to obtain the labels of the 
required inward edges in $E(T)$ (see Example~\ref{ex:in}), where 
$i\in \{1, 2, \ldots, 2n\}$, $j \in \{1, 2, \ldots, d\}$ and $C$ is a large number.
\begin{align}
p_i &= 
\begin{cases}
1 &\text{~if~} t_i \leq m,\\
0 &\text{~otherwise},
\end{cases} \label{eqp3}\\
p'_i  &= \max(\sum_{k=1}^{i} p_k  - C\delta(p_i, 0) , 0)  \label{eqp'3}, \\
p''_i  &= p'_i  +  \max(2n - C(1-\delta(p'_i, 0) ), 0)  \label{eqp''3}, \\
q_{ji}  &= \delta(p''_i, x_j),  \label{eqq3}\\
r'_{ji}  &= t_i \cdot q_{ji}. \label{eqr'3}
\end{align}
Eq.~(\ref{eqp3}) outputs a binary variable $p_i$ which is 1 if and only if the $i$-th entry of $E(T)$ is the label of an inward edge.  
Eq.~(\ref{eqp'3}) identifies the DFS index of only inward edges in $E(T)$ (see Example~\ref{ex:in}).
That is $p_i' = \ell \neq 0$ if and only if the $i$-th entry of $E(T)$ corresponds to the $\ell$-th inward edge in the directed tree. 
Eq.~(\ref{eqp''3}) replaces $p_i' = 0$ by $2n$ to ignore the root case. 
The variable $q_{ji} = 1$ if and only if $p_i' = x_j$ in Eq.~(\ref{eqq3}), and
Eq.~(\ref{eqr'3}) is used to identify the labels of the desired inward edges. 
Note that all these equations involve the maximum function or $\delta$ function which can be simulated 
by the ReLU activation function based on \comblue{Proposition~1 by Ghafoor and Akutsu~\cite{MT2024}.}
Therefore we can construct an eight-layer neural network with ReLU as an activation function with size 
$\mathcal{O}(dn)$ and constant depth that can identify the labels of the 
inward edges of non-zero $x_j$. 
\end{proof}
A demonstration of Lemma~\ref{thm:inward} is given in Example~\ref{ex:in}. 
\begin{example}\label{ex:in}
Consider the tree $T$ shown in Fig.~(\ref{fig:ET})(a) with 
$E(T) =3, 2, 7, 2, $ $4, 9, 7, 4, 9, 8$, and the random sequence
$x =1, 3, 0 $. 
We wish to identify the labels of the inward edges of 
$x_j \neq0$ in $E(T)$. 
For $x_1 = 1$ and $x_2 = 3$, the labels of the inward edges are 
$3$ and $2$, resp., whereas $x_3 = 0$ does not correspond to 
any inward edge, and therefore it is ignored. 
The variables that are used in the process of obtaining the required labels by using Lemma~\ref{thm:inward} are discussed and illustrated in Table~\ref{tb:inward} and Fig.~(\ref{fig:inward}).
\begin{table*}[h!]
\centering
\caption{The variables, their meaning and example values used in Lemma~\ref{thm:inward}. }
\begin{tabular}{|c|p{7.5cm}|l|} 
\hline
Variable & \hspace{3cm}Meaning &  \hspace{2.2cm} Value \\ \hline \hline
$x_j$ & Specify the inward edge of $x_j$ of which the label is required. & 
\makecell[l]{$x = 1, 3, 0$\\
(Fig.~\ref{fig:ET}(a))} \\ \hline 
$p_i$   &  A binary variable which is one if there is an inward edge at the $i$-th position of $E(T)$.   & \makecell[l]{$p = [1, 1, 0, 1, 1, 0, 0, 1, 0, 0]$\\
 (Fig.~\ref{fig:inward}(a))}  \\ \hline
 $p_i'$   & The DFS index of the inward edge, among all the inward edges, that is at the $i$-th position of $E(T)$.   
 	& \makecell[l]{$p' = [1, 2, 0, 3, 4, 0, 0, 5, 0, 0]$ \\
 	 (Fig.~\ref{fig:inward}(b))}   \\ \hline
$p''_i$   & Replacing $p'_i = 0$ with $2n = 10$ in $p'$ to ignore the root case. & \makecell[l]{$p'' =$\\ $[1, 2, 10, 3, 4, 10, 10, 5, 10, 10]$}  \\ \hline
$q_{ji}$ & A binary variable which identifies the position $i$ of the non-zero input $x_j$ in $E(T)$, i.e., 
$q_{ji} = 1$ if and only if $p''_i = x_j$. & 
	\makecell[l]{$q_{1,1}=q_{2,4}=1$, \\other variables are zero\\ 
	(Fig.~\ref{fig:inward}(c))  }\\ \hline
$r'_{ji}$ & The required label of the inward edge of $x_j$. &
	\makecell[l]{$r'_{1,1}=3, r'_{2,4}=2$,\\ other variables are zero\\
	(Fig.~\ref{fig:inward}(c))}\\ \hline
\end{tabular}
\label{tb:inward}
\end{table*}
\begin{figure*}[ht!]
	\centering
	\includegraphics[scale = 0.6]{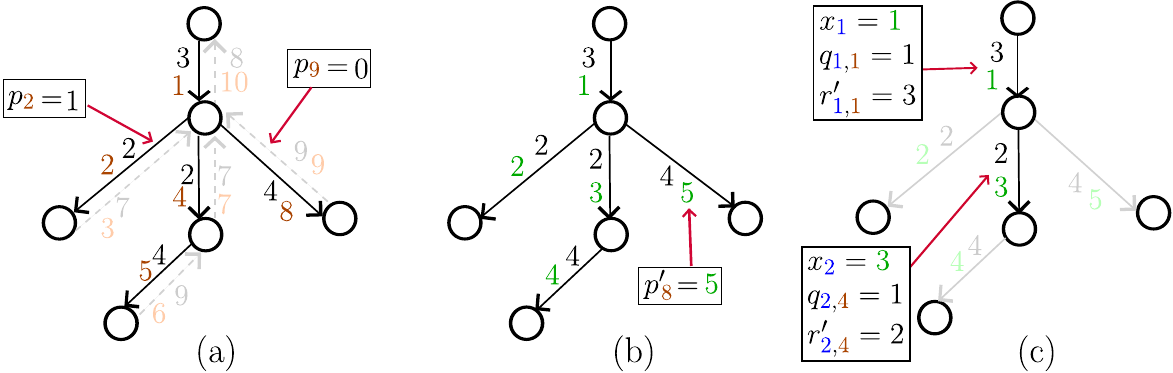}
	\caption{Illustrations of the variables used Eqs.(\ref{eqp3})-(\ref{eqr'3}) in Lemma~\ref{thm:inward}: 
	(a) The variable $p_i$ which is 1 for the inward (black) edges and 0 for the outward (gray) edges in the directed tree corresponding to the tree 
	$T$ given in Fig.~\ref{fig:ET}(a), e.g., $p_2 = 1$ (resp., 
	$p_9 = 0$) as there is an inward edge (resp., outward edge) with the DFS index 2 
	(resp., 9);
	(b) The variable $p'_i$ (green), e.g., $p'_8 = 5$ means that 
	the inward edge of $5$ has the DFS index 8 in (a); 
	(c) For a fixed DFS index $i$, the variable $q_{ji} =1$ for some $j$ 
	(black inward edge), and 
	$q_{ji} =0$  for all $j$ (gray inward edges), e.g., 
	for the DFS index $i  =4$, we have $j = 2$ such that $q_{2,4} = 1$ and 
	thus the inward edge with the DFS index $4$ is depicted in black,  
	whereas for $i = 8$ there does not exist any $j$ such that $q_{j8} = 1$, and 
	so the inward edge with the DFS index $8$ is depicted in gray. 
	The dark edges are the
	desired inward edges of which labels are required. The labels of these edges are 
	stored by the 
	variable $r'_{ji}$, e.g., $r'_{2,4} = 2$ means that the desired inward edge 
	specified by $x_2$ has the DFS index $4$ and label $2$. }\label{fig:inward}
\end{figure*}
\end{example}

Proposition~\ref{pro:e-bar1} gives a necessary and sufficient condition for the $i$-th entry to be the outward edge of an inward edge at $\ell$-th position in an Euler string $E(T)=t_1, t_2, \ldots, t_{2n}$. 
The condition essentially depends on the number of inward edges and outward edges between the two given positions $i$ and $\ell$.
Before going into the details, for any two positions ${i}$ and ${\ell}$, with $1 \leq \ell < i \leq 2n$, we denote by $\rm{in}_{\ell, i}$ (resp., $\rm{out}_{\ell, i}$), the total number of 
inward (resp., outward) edges among the entries $t_k$, $\ell < k < i$, where 
$t_k$ is the $k$-th entry of $E(T)$.  
Consider the tree $T$ given in Fig.~\ref{fig:ET}(a) and its Euler string $E(T)$ given in Fig.~\ref{fig:ET}(c), for 
$i = 7$, $\ell = 2$, $\rm{in}_{2,7} = 2$ as there are two inward edges 
$t_4, t_5$ and $\rm{out}_{2,7} = 2$ as there are two outward edges  
$t_3, t_6$ (see Fig.~\ref{fig:inofout}(a)).
So in this case $\rm{in}_{2,7} = \rm{out}_{2,7}$. 
Similarly, for $i = 9, \ell = 3$, $\rm{in}_{3,9} = 3 > 2 = \rm{out}_{3,9}$ (see Fig.~\ref{fig:inofout}(b)), for $i = 9, \ell = 5$, $\rm{in}_{5,9} =1 < 2 = \rm{out}_{5,9}$ (see Fig.~\ref{fig:inofout}(c)), and for $i = 7, \ell = 4$, $\rm{in}_{4,7} =1 = \rm{out}_{4,7}$ (see Fig.~\ref{fig:inofout}(d)).
\begin{figure}[t!] 
	\centering
	\includegraphics[scale = 0.6]{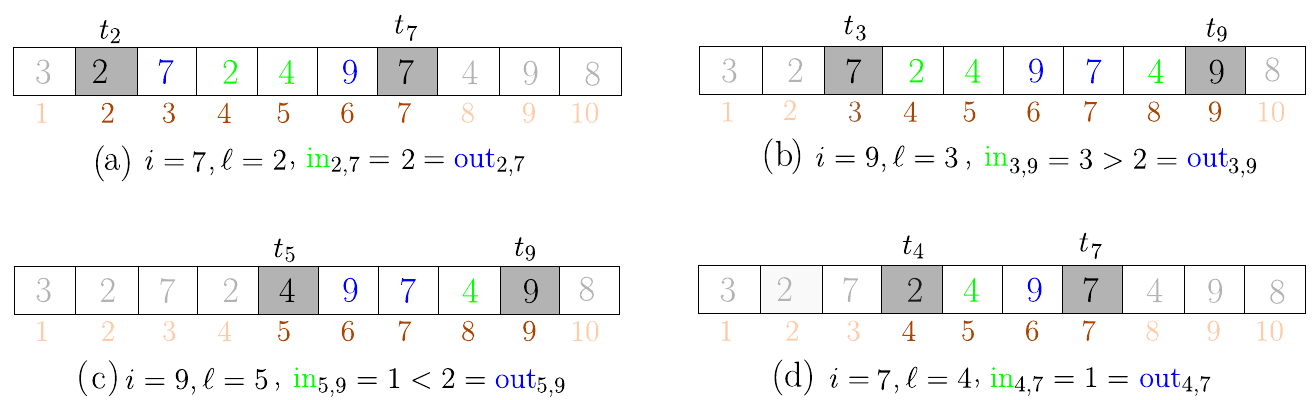}
	\caption{(a)-(c)An illustration of the number of inward edges (green) (resp., outward edges (blue)) for the edges $t_i, i= 7, 9$  and $t_{\ell}, \ell = 2, 3, 5$ which are depicted in gray boxes; 
	(d) An example in which Proposition~\ref{pro:e-bar1} holds.  
	}\label{fig:inofout}
\end{figure}	
\pagebreak
\begin{proposition}
\label{pro:e-bar1}
Let $E(T) = t_1, t_2, \ldots, t_{2n}$ denote the Euler string of a tree $T$ with $n$ edges and  $t_i \in \Sigma =\{1, 2, \ldots, m\}$.
Then $t_i$,  $1 \leq i \leq 2n$,  is an outward edge of $t_{\ell}$ if and only if 
(i)-(iv) hold
\begin{enumerate}
\item[(i)]$ {\ell} \in [1, i-1]$,
\item[(ii)]$ t_i= t_{\ell}+m$,  
\item[(iii)]$\rm{in}_{\ell, i} = \rm{out}_{\ell, i}$, and 
\item[(iv)]${\ell}$ is the largest index that satisfies (i)-(iii).
\end{enumerate}
\end{proposition}
\begin{proof}
We know that the Euler string follows the DFS. 
Therefore the DFS index of all descendant edges of a given inward edge appear between the inward edge and its outward edge from which the result follows.
\end{proof}
\noindent
A demonstration of Proposition~\ref{pro:e-bar1} is given in Example~\ref{ex:desc}. 
\begin{example}\label{ex:desc}
Consider the tree $T$ given in Fig.~\ref{fig:ET}(a) and its Euler string $E(T)$ shown in
Fig.~\ref{fig:ET}(c). 
We want to determine if the edges $t_i$, $i = 4, 7, 9$ in $E(T)$ are outward edges of some inward edges by using Proposition~\ref{pro:e-bar1}. 
We see that Proposition~\ref{pro:e-bar1}(ii) does not hold for $i = 4$ and any 
$\ell \in [1, 3]$, therefore $t_4$ is not an outward edge of any inward edge, which is consistent with the fact.  
For $i = 7$, Proposition~\ref{pro:e-bar1}(ii) and~(iii) are satisfied for both $\ell = 2, 4$.  
For $\ell = 2$ (resp., $\ell = 4$), $\rm{in}_{2,7} = 2=\rm{out}_{2,7}$ (resp., $\rm{in}_{4,7} =1= \rm{out}_{4,7}$) (see Fig.~\ref{fig:inofout}(a)). That is $t_2$ and $t_4$ are both candidate inward edges of $t_7$. By using Proposition~\ref{pro:e-bar1}(iv), $t_7$ is the outward edge of $t_4$(see Fig.~\ref{fig:inofout}(d)). 
For $i=9$, we see that $t_5$ and $t_8$ are the only edges that satisfy 
Proposition~\ref{pro:e-bar1}(i)-(ii), but $t_5$ does not satisfy the Proposition~\ref{pro:e-bar1}(iii). Therefore 
$t_8$ is the inward edge of $t_9$ implying that $t_9$ is an outward edge.
\end{example} 
%
The existence of a ReLU network to identify the 
positions and labels of outward edges in an Euler string based on   
Proposition~\ref{pro:e-bar1} is discussed in Lemma~\ref{thm:ebar4}. 
\begin{lemma}
\label{thm:ebar4}
Let $T$ be a tree with $n$ edges, 
and $ x = x_1, x_2, \ldots, x_{d}$ be a random sequence of integers 
over the interval $[0, n]$.  
Then there exists a ReLU network with size 
$\mathcal{O}(dn^2)$ and constant depth that can identify the position and label
of outward edge of each non-zero input $x_j$ in the Euler string $E(T)$. 
\end{lemma}
\begin{proof}
Let $E(T) = t_1, t_2, \ldots, t_{2n}$. 
The proof completes by expressing the conditions of Proposition~\ref{pro:e-bar1}
in terms of ReLU activation function. 
Proposition~\ref{pro:e-bar1} requires the labels of the inward edges of each $x_j \neq 0$ which can be computed by using 
Eqs.(\ref{eqp3})-(\ref{eqr'3}) of Lemma~\ref{thm:inward}. 
Then the positions and labels of the 
required outward edges in $E(T)$ can be obtained by using the following system of equations (see Example~\ref{ex:out}), where 
$i, \ell\in \{1, 2, \ldots, 2n\}$, $j \in \{1, 2, \ldots, d\}$ and $C$ is a large number.

\begin{align}
r_i  &= t_i \cdot p_i,  \label{eqr3}\\
s_i &= t_i - r_i, \label{eqs3}\\
%
v_{{\ell}i} &= \begin{cases}
0 &\text{~if~}  i\leq {\ell},\\
\max ( \delta(s_i, r_{\ell} + m) - \\
	~~C ( \sum_{k={\ell}+1}^{i-1} H(s_k-1)  - \\ 
~~\sum_{k={\ell}+1}^{i-1} H(r_{k} -1)), 0 ) & \text{otherwise}, 
\end{cases} \label{eqv3}\\
v'_{{\ell}i}  &= \delta(v_{{\ell}i} , 1), \label{eqv'3}\\
w_{{\ell}i}&= 
\max(v'_{{\ell}i} - \sum_{k={\ell}+1}^{i-1} v'_{ki} , 0), \label{eqw3}\\
w'_{j{\ell}i}&= 
\begin{cases}
0 &\text{~if~}  i\leq {\ell},\\
\max(\delta(s_i, r'_{j \ell} + m) - \sum_{k=1, k \neq {\ell}}^{2n} w_{ki}, 0) &\text{~otherwise}, \label{eqw'3}\\
\end{cases} 
\end{align}
\begin{align}
z_{j}&=i \cdot \sum_{i=1}^{2n} \sum_{{\ell}=1}^{2n} w'_{j{\ell}i}  , \label{eqz3}\\
z'_{ji}&= t_i \cdot \sum_{{\ell}=1}^{2n} w'_{j{\ell}i} . \label{eqz'3}
\end{align}
The non-zero variable $r_i$ (resp., $s_i$) in Eq.~(\ref{eqr3}) (resp., Eq.~(\ref{eqs3})) stores the label of the inward edge (resp., outward edge) at the $i$-th position of $E(T)$.  
Eqs.~(\ref{eqv3}) and~(\ref{eqv'3}) encode 
Proposition~\ref{pro:e-bar1}(ii) and (iii), and 
Eq.~(\ref{eqw3}) encodes Proposition~\ref{pro:e-bar1}(iv). 
In Eq.~(\ref{eqw'3}), $w'_{j\ell i} = 1$  if and only if 
$s_i = r'_{j\ell}+ m$ and $r'_{j\ell}$ is the largest index with this property, 
i.e., all conditions of Proposition~\ref{pro:e-bar1} are satisfied. 
Eq.~(\ref{eqz3}) (resp., Eq.~(\ref{eqz'3})) determines the positions (resp., labels) in $E(T)$ of the desired outward edges. 
These equations involve the maximum function, $\delta$ function and Heaviside function which can be simulated by the ReLU activation function based on \comblue{Proposition~1 by Ghafoor and Akutsu~\cite{MT2024} and Theorem~1 by Kumano and Akutsu~\cite{KA2022}.} 
Therefore we can construct a twelve-layer neural network with ReLU as an activation function with size 
$\mathcal{O}(dn^2)$ and constant depth that can identify the positions and labels of the 
desired outward edges of non-zero $x_j$. 
\end{proof}

A demonstration of Lemma~\ref{thm:ebar4} is given in Example~\ref{ex:out}.
\begin{example}\label{ex:out}
Reconsider the tree $T$ shown in Fig.~(\ref{fig:ET})(a) with 
$E(T) =3, 2, 7, 2, $ $4, 9, 7, 4, 9, 8$, and the random sequence
$x =1, 3, 0 $. 
We wish to identify the positions and labels of the outward edges of 
$x_j \neq0$ in $E(T)$. 
Note that the positions of outward edges are their DFS indices in the directed tree corresponding to $T$ as demonstrated in Figs.~(\ref{fig:ET})(b) and~(c). 
For $x_1 = 1$ (resp., $x_2 = 3$), the positions and labels of the outward edges are 
$10$ and $8$ (resp., $7$ and $7$), resp., whereas $x_3 = 0$ does not correspond to 
any outward edge, and therefore it is ignored. 
We discuss the variables used in Lemma~\ref{thm:ebar4} to get the required positions and labels as follows. 
The variables $p_i, p_i', p_i'', q_{ji}, r_{ji}'$ used in Eqs.~(\ref{eqp3})-(\ref{eqr'3})
are discussed, in detail, in Example~\ref{ex:in} and Fig.~\ref{fig:inward}.
We discuss the variables of Eqs.~(\ref{eqr3})-(\ref{eqz'3}). 
An illustration of these variables is given in Fig.~(\ref{fig:outward}). 
In the rest of the discussion, more than one subscripts are separated by the commas to avoid confusion.
\begin{figure*}[t!]
	\centering
	\includegraphics[scale = 0.6]{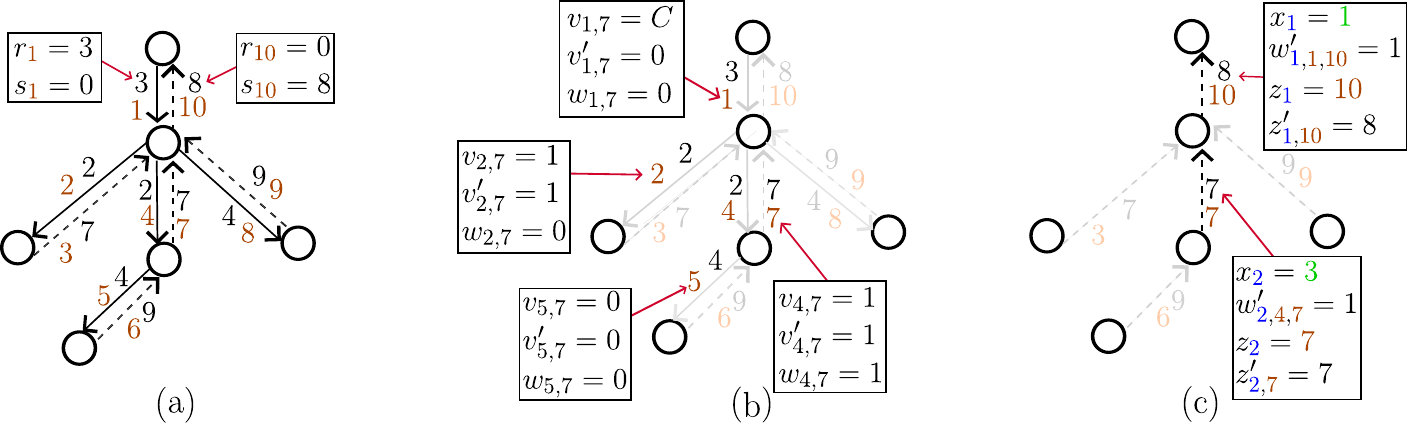}
	\caption{	An illustration of the variables used in Eqs.~(\ref{eqr3})-(\ref{eqz'3}) of Lemma~\ref{thm:ebar4}
	to identify the positions and labels of the desired outward edges. 
	}\label{fig:outward}
\end{figure*}

\begin{longtable}{c p{16cm}} 
\addtocounter{table}{-1}
$x_j$ &Specify the outward edge of $x_j$ of which the position and label are required. 
In this case $x = 1, 3, 0$, which is illustrated in Fig.~\ref{fig:ET}(a). \\   
\multicolumn{2}{l}{$p_i, p_i', p_i'', q_{ji}, r_{ji}'$ are explained in Example~\ref{ex:in}, Table~\ref{tb:inward} and Fig.~\ref{fig:inward}.}\\
$r_i$ &     The label of the inward edge of $i$, if it exists, e.g., 
 in Fig.~\ref{fig:outward}(a), $r_1 = 3$ (resp., $r_{10} = 0$) as there exists (resp., does not exist) an 
	inward edge of $1$ (resp., $10$). The label of the inward edge of $1$ is $3$. 
	The values of these variables are listed in $r = [ 3, 2, 0, 2, 4, 0, 0, 4, 0, 0 ] $.
   \\
 $s_i$   &  The label of the outward edge of $i$, if it exists, e.g.,  
 in Fig.~\ref{fig:outward}(a), $s_{10} = 8$ (resp., $s_1 = 0$) as there exists (resp., does not exist)
	an outward edge of $10$ (resp.,  $1$). The label of the outward edge of $10$ is $8$.  Similarly, we get
	 $s = [ 0, 0, 7, 0, 0, 9, 7, 0, 9, 8 ]$.\\   
$v_{\ell i}$ &  A variable which can take a value from $\{1, C, 0\}$: 
$v_{\ell i} = 1$ if the conditions of Proposition~\ref{pro:e-bar1}(ii) and (iii) are satisfied, e.g., $i = 7, \ell = 2, 4$, 
$\rm{in}_{\ell i} = \rm{out}_{\ell i}$ (see Example~\ref{ex:desc} and Fig.~\ref{fig:inofout}(a));
 $v_{\ell i} = C$ if Proposition~\ref{pro:e-bar1}(iii) is violated with the number of number of inward edges greater than the number of outward edges, i.e., 
 $\rm{in}_{\ell i} > \rm{out}_{\ell i}$ holds, e.g., 
 $v_{3, 9} = C$ as $\rm{in}_{3, 9} > \rm{out}_{3, 9}$ (see Fig.~\ref{fig:inofout}(b)); 
 $v_{\ell i} = 0$ if Proposition~\ref{pro:e-bar1}(iii) is violated with the number of inward edges less than that of outward edges, i.e.,
 $\rm{in}_{\ell i} < \rm{out}_{\ell i}$ holds, e.g., 
 $v_{5, 9} = 0$ as  $\rm{in}_{5, 9} < \rm{out}_{5, 9}$ (see Fig.~\ref{fig:inofout}(c)).
 The values $v_{2, 7} = v_{4,7} = 1$, $v_{5, 7} = 0$ and $v_{1, 7} = C$ are illustrated in Fig.~\ref{fig:outward}(b). 
	Thus we have,
	  {$v_{1,10}= v_{2,3}= v_{2,7}= v_{4,7}= v_{5,6}=v_{8,9}=1$,   $
v_{1,3}=v_{1,5}=v_{1,6}=v_{1,7}=v_{1,9}=v_{2,6}= v_{3,5}=v_{3,6}=v_{3,7}=v_{3,9}=v_{4,6}=v_{7,9}=C$},  and other variables are zero.\\
$v'_{\ell i}$ &  Replacing $C$ with $0$ in $v_{\ell i}$, e.g., 
$v_{2,6} = C$, and therefore $v'_{2,6} = 0$ (see Fig.~\ref{fig:outward}(b)). 
Thus $v'_{1,10}= v'_{2,3}= v'_{2,7}= v'_{4,7}= v'_{5,6}=v'_{8,9}=1$, and other variables are zero.  \\
$w_{\ell i}$  & A binary variable which is one when $\ell$ is the largest number that satisfies Proposition~\ref{pro:e-bar1}(i)-(iii), i.e., Proposition~\ref{pro:e-bar1}(iv)is satisfied, e.g., 
$v'_{2,7} = v'_{4,7} = 1$ are the only non-zero variables for $i = 7$. 
Since $4$ is largest among such variables, we have $w_{4,7} = 1$ (see Fig.~\ref{fig:outward}(b)).  
Similarly, $w_{1,10}=w_{2,3}=w_{4,7}=w_{5,6}=w_{8,9}=1$,  
and other variables are zero.    \\  
$w'_{j\ell i}$ & A binary variable to identify the desired outward edges corresponding to $x_j \neq 0$. More precisely $w'_{j\ell i}$ is one when $t_i$ is the outward edge of the inward edge $t_{\ell}$ (Proposition~\ref{pro:e-bar1} is satisfied), and $t_{\ell}$ is the inward edge of $x_j$, e.g., 
$w_{2,4,7} = 1$ because $t_7$ is the outward edge of the inward edge $t_4$ which corresponds to 
 $x_2 = 3$  (see Fig.~\ref{fig:inward}(c)). 
In Fig.~\ref{fig:outward}(c) the outward edges that have non-zero (resp., zero) value of $w'_{j\ell i}$ are depicted by dark (resp., gray) edges. Thus, 
   $w'_{1,1,10}=w'_{2,4,7}=1$, and  
other variables are zero.\\
$z_j$  & Position $i$ of the outward edge of $x_j \neq 0$. 
In Fig.~\ref{fig:outward}(c), $z_1=10$ and $z_2=7$ means that the position of the outward edges of $x_1$ (resp., $x_2$) is $10$ (resp., $7$). Thus 
  $z =[10, 7, 0]$. \\
$z'_{ji}$ &  Label of the outward edge of $x_j \neq 0$ which is at the position $i$. 
In this case 
  $z'_{1,10}= 8, z'_{2,7}= 7$,   
and other variables are zero  
(see Fig.~\ref{fig:outward}(c)).
\end{longtable}
\end{example}
\section{TS$_d$-generative ReLU}\label{sec:ESGR}
Let $T$ be a tree with $n+1$ nodes and labels from 
$\Sigma = \{1, 2, \ldots, m\}$, Euler string 
$E(T)=t_1, t_2, \ldots, t_{2n}$, and a non-negative integer $d$. 
We define the {\em TS$_d$-generative ReLU} to be a 
ReLU neural network with $2d$ input nodes 
$ x = x_1, x_2, \ldots, x_{2d}$ over $\{0, \ldots, n\}$, 
and $2n$ output nodes 
$u = u_1, u_2, \ldots, u_{2n}$ over $\Sigma$ such that all Euler strings 
$u$ of trees with the tree edit distance at most $2d$ 
from $E(T)$ can be obtained by the substitution of appropriate $x_{j+d}$ and $x_{j+d} + m$ at 
the inward edge and outward edge of $x_{j} \neq 0$, respectively. 
In this context $x_1, \ldots, x_d$ represents the inward edges for the substitution operations, while $x_{d+1}, \ldots, x_{2d}$ denotes the values to be substituted during these operations. The generated Euler strings have distance at most $2d$ as substitution at root and repeated terms in $x$ will be ignored. 
An illustration of such a network is given in Fig.~\ref{fig:sub_N}, 
where $m = 5$, $d = 3$, $x=1, 3, 0, 5, 1, 2$,
means that the labels of the inward edges (resp., outward edges) of $x_1 = 1, x_2 = 3$ are substituted by $x_{4} = 5, x_{5} = 1$ (resp., $10, 6$). 
Since the entry $x_3 = 0$ and $x_6=2$ corresponds to the root and its label, resp., and so they are ignored.
\begin{figure}[h!]
	\centering
	\includegraphics[scale = 0.6]{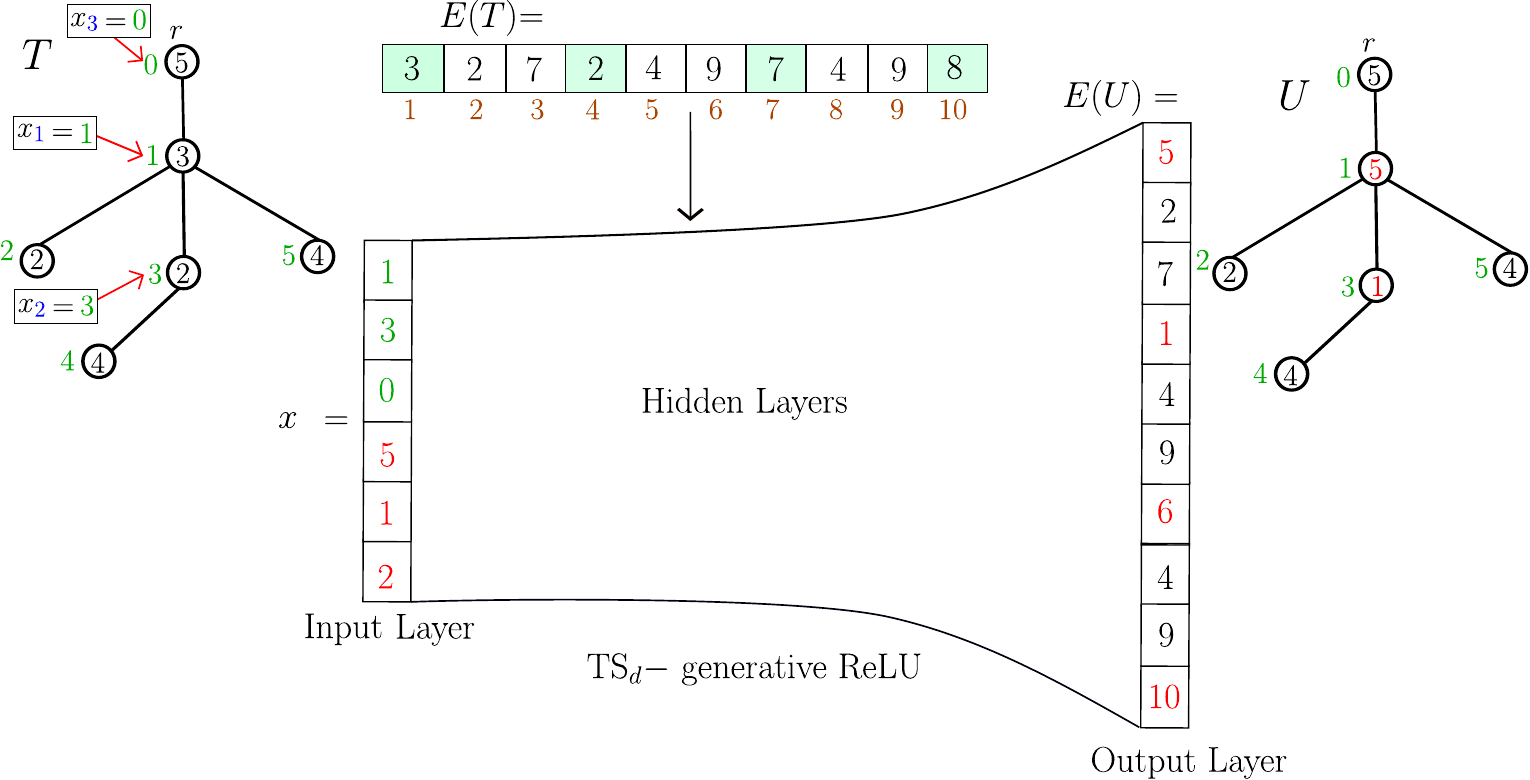}
	\caption{	An illustration of a TS$_d$-generative ReLU with the input layer 
$x=1, 3, 0, 5, 1, 2$, $E(T)=3, 2, 7, 2, 4, 9, 7, 4, 9, 8$ and output layer $u= E(U)=5, 2, 7, 1, 4, 9, 6, 4, 9, 10$. 
The substitution operations on $E(T)$ and $E(U)$ are depicted with green boxes and red values, respectively.
	}\label{fig:sub_N}
\end{figure}
\noindent
 
 \comblue{The existence of TS$_d$-generative ReLU network is discussed in Theorem~\ref{thm:Esnn}.}
\begin{theorem}
\label{thm:Esnn}
For a rooted ordered tree $T$ of size $n+1$ with a label set
$\Sigma = \{1, 2, \ldots, m\}$, and a non-negative integer $d$, 
there exists a TS$_d$-generative ReLU network with size 
$\mathcal{O}(dn^2)$ and constant depth.
\end{theorem}
\noindent
A proof and an explanation of each variable of Theorem~\ref{thm:Esnn} is given in 
Example~\ref{exa:Sub} in Appendix~\ref{sec:app}.
\section{TD$_d$, TI$_d$-generative ReLU}\label{sec:EdGR}
Let $T$ be a tree with $n+1$ nodes, labels from 
$\Sigma = \{1, 2, \ldots, m\}$ and Euler string 
$E(T)=t_1, t_2, \ldots, t_{2n}$, and a non-negative integer $d$. 
We define the {\em TD$_d$-generative ReLU} to be a 
ReLU neural network 
that can generate all Euler strings over $\Sigma$ with the tree edit distance at most $2d$ from $E(T)$ by deleting from $E(T)$, the appropriate inward and outward edges of 
$x = x_1, x_2,  \ldots, x_d$, over $\{0, \ldots, n\}$. 
The generated Euler strings have distance at most $2d$ as deletion of root and repeated terms in $x$ will be ignored. 
Since the network does not always delete exactly $2d$ elements from the input Euler string, while a neural network must have a fixed number of nodes in each layer, we pad $E(T)$ with $2d$ number of $B$s where $B \gg m$, to ensure a fixed output dimension. To delete a vertex, both its inward and outward edges must be deleted.   Therefore, for each $x_j \neq 0$, the network deletes two elements of the input Euler string corresponding to the inward and outward edges of $x_j$, and for each $x_j = 0$, it deletes two padding symbols $B$.
We call such a string a {\em padded Euler string
We call such a string a {\em padded Euler string}. 
In this setting, we fix $2n$ output nodes $y = y_1, \ldots, y_{2n}$ of the network, where 
$y_1, \ldots, y_{2n-2d'}$, $d' \leq d $ is the Euler string of some tree $U$ with the tree edit distance 
$2d$ from $E(T)$ by deleting $d'$ inward and outward edges of non-zero entries of $x$, and the remaining $2d'$ entries $y_{2n-2d'+1}, \ldots, y_{2n}$ are $B$s. 
An illustration of a TD$_d$-generative ReLU is given in Fig.~\ref{fig:Del_N}, 
where $m = 5$, $d = 3$, $x=1, 3, 0$, and $T$ is given in Fig.~\ref{fig:ET} with $E(T) = 3, 2, 7, 2, 4, 9, 7, 4, 9, 8$, and the padded $E(T)$ with six $B$s. 
The network will delete the inward and outward edges of $1$ and $3$ as depicted in the figure, and delete two $B$s corresponding to $0$. 
The resultant string is $y=2, 7, 4, 9, 4, 9, B, B, B, B$, and by removing 
$B$s from $y$ we get the desired Euler string $E(U)=2, 7, 4, 9, 4, 9$ of the tree $U$ as shown in the Fig.~\ref{fig:Del_N}.

\begin{figure}[h!]
	\centering
	\includegraphics[scale = 0.6]{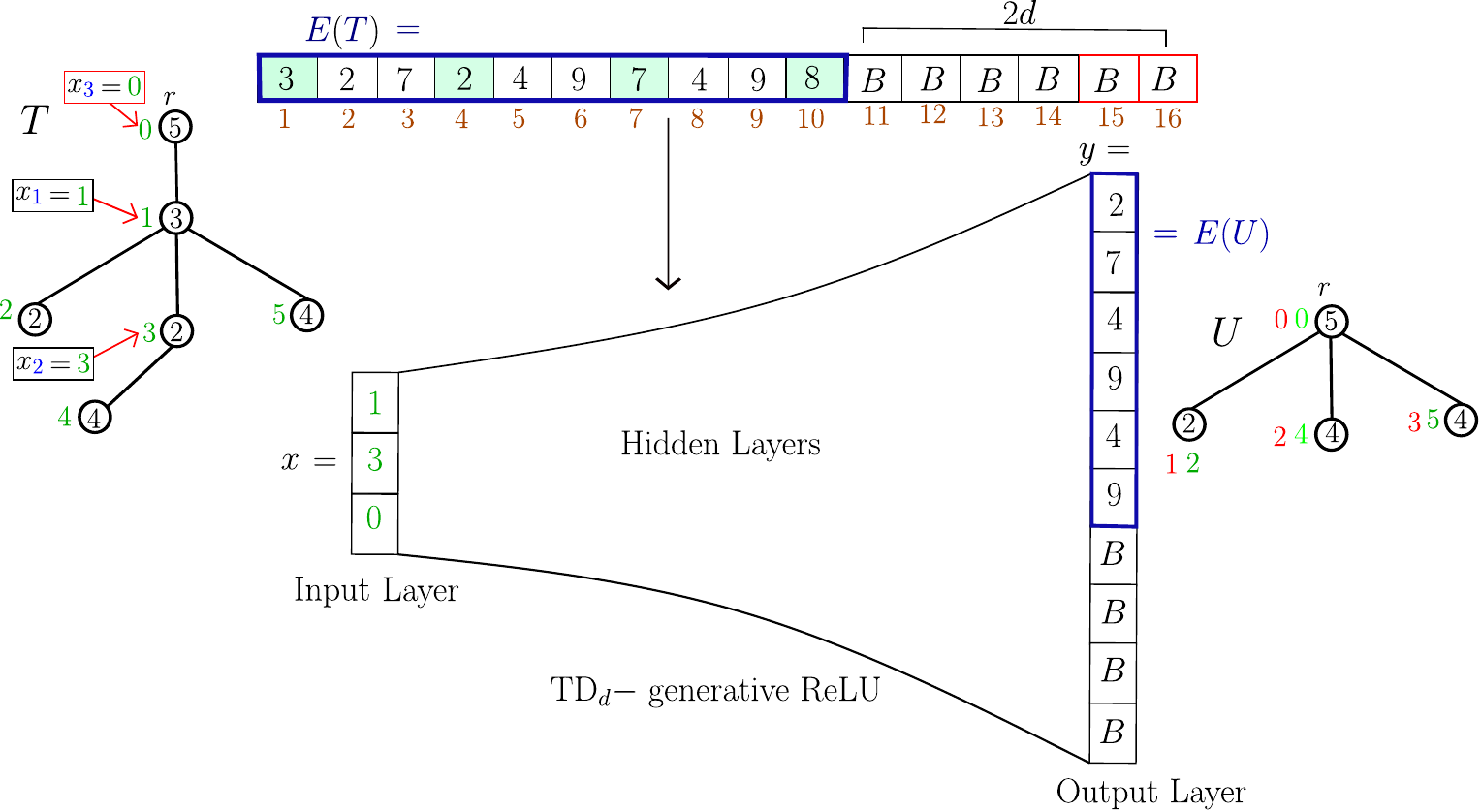}
	\caption{An illustration of a TD$_d$-generative ReLU with the input layer 
$x=1, 3, 0$, padded Euler string $3, 2, 7, 2, 4, 9, 7, 4, 9, 8, B, B, B, B, B, B$ and the output layer $2, 7, 4, 9, 4, 9, B, B, B, B$. 
By deleting $B$s we can get the resultant string $E(U)=2, 7, 4, 9, 4, 9$ obtained by deleting the inward and outward edges of $x = 1, 3$. 
	}\label{fig:Del_N}
\end{figure}
\noindent
}
\comblue{The existence of TD$_d$-generative ReLU network is discussed in Theorem~\ref{thm:Ednn}.}

\begin{theorem}
\label{thm:Ednn}
For a rooted ordered tree $T$ of size $n+1$ with nodes from 
$\Sigma = \{1, 2, \ldots, m\}$, and a non-negative integer $d$, 
there exists a TD$_d$-generative ReLU network with size 
$\mathcal{O}(n^2)$ and constant depth.
\end{theorem}
\noindent
A proof and an explanation of each variable of Theorem~\ref{thm:Ednn} is given in 
Example~\ref{exa:Del} in Appendix~\ref{sec:app}.

We define {TI$_d$-generative ReLU} as follows.
Let $T$ be a tree with $n+1$ nodes and labels from 
$\Sigma = \{1, 2, \ldots, m\}$, Euler string 
$E(T)$, and a non-negative integer $d$. 
We define the {\em TI$_d$-generative ReLU} to be a 
ReLU neural network with $4d$ input nodes 
$ x = x_1, x_2, \ldots, x_{4d}$ over $\{0, \ldots, n\}$, 
and $2n + 2d$ output nodes 
$u = u_1, u_2, \ldots, u_{2n +2d}$ over $\Sigma$ such that all Euler strings 
$u$ of trees with the tree edit distance exactly $2d$ 
from $E(T)$ can be obtained by 
the insertion of appropriate child nodes $x'_{j}$ of $x_j$ with 
inward edges and outward edges of labels 
$x_{j+3d}$ and $x_{j+3d} + m$, resp., of $x'_{j}$,
and setting the appropriate ($x_{j+d}$, $x_{j+d}+1$, $x_{j+d}+2$, \ldots, $x_{j+2d}$)-th
children of the node $x_j$ as the children of $x'_{j}$. It means $x'_{j}$ becomes the parent of a sequence of children that starts from $x_{j+d}$ and ends at $x_{j+2d}$. 
In this study, we do not consider inserting new nodes as children of a newly inserted node.
In this context, $x_1, \ldots, x_d$ represent the nodes for the insertion operations,  $x_{1+d}, \ldots, x_{d+d}$ and 
$x_{1+2d}, \ldots, x_{d+2d}$ represent the lower and upper bounds 
to determine the subsequences of children that will be set as the children of the inserted nodes, and 
$x_{1+3d}, \ldots, x_{d+3d}$ represents the labels to be inserted.
For convenience, we denote  
$x_1, \ldots, x_{d}$ , 
$x_{1+d}, \ldots, x_{2d}$, 
$x_{1+2d}, \ldots, x_{3d}$, and 
$x_{1+3d}, \ldots, x_{4d}$
by 
$x^{1}$, $x^{2}$, $x^{3}$, and $x^{4}$, respectively. 
Note that lower and upper bounds on the number of children must not exceed the total number of children of a node $x_j^1$. Let $D(x_j^1)$ denote the total number of children of a node $x_j^1$ then we require $x_j^2, x_j^3 \leq D(x_j^1)$. Due to the random nature of $x$, some of the lower and upper bounds of children may not be valid, and thus need to be refined to perform appropriate insertion operations.  
Such invalid bounds, their refinements and appropriate insertions are listed in Table~\ref{tab:reset}, where $D(x_j^1)$ denotes the number of children of the node 
$x_j^1$. 
\begin{figure}[h!]
	\centering
	\includegraphics[scale = 0.6]{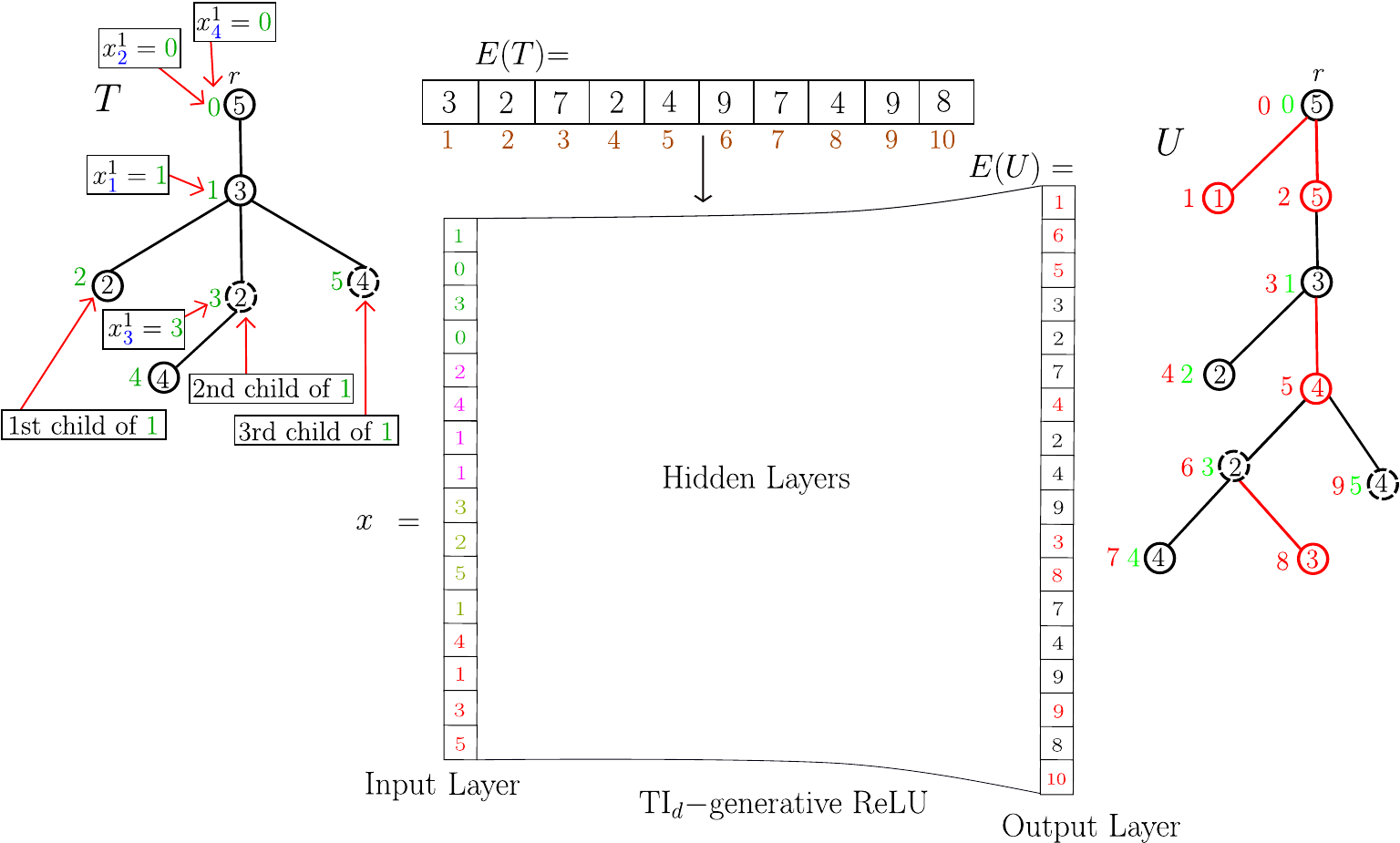}
	\caption{	An illustration of a TI$_d$-generative ReLU with the input layer 
$x=1, 0, 3, 0, 2, 4, 1, 1, 3, 2, 5, 1, 4, 1, 3, 5$, $E(T)=3, 2, 7, 2, 4, 9, 7, 4, 9, 8$ and output layer $u= E(U)=1, 6, 5, 3, 2, 7, 4, 2, 4, 9, 3, 8, 7, 4, 9, 9, 8, 10$. 
The insertions are depicted in red in $U$. 
	}\label{fig:in_N}
\end{figure}
\begin{table*}[h!]
\setlength{\tabcolsep}{0pt}
\centering
\caption{ Invalid bounds of children, their refinements and appropriate insertions.}
\begin{tabular}{|c|c|c|c|}\hline
S. no. & Invalid bounds & Refinements & Insertions\\\hline
(i)& $D(x_j^1) < x_j^2$ &  $x_j^2:=0$ & Insert a leaf before the first child of $x_j^1$\\
(ii)& $D(x_j^1) < x_j^3$ &  $x_j^3:=0$ & Insert a leaf after the $x_j^2$-th child of $x_j^1$\\
(iii)& $x_j^2 > x_j^3$ &  $x_j^3:=0$ & Insert a leaf after the $x_j^2$-th child of $x_j^1$\\
(iv)& $x_j^1=  x_k^1, j< k, x_j^3 > x_k^2$ &  $x_j^3:=0$ & Insert a leaf after the $x_j^2$-th child of $x_j^1$\\
(v)& $x_j^1=  x_k^1, j> k, x_j^2 = x_k^3$ &  $x_j^3:=0$ & Insert a leaf after the $x_j^2$-th child of $x_j^1$\\
(vi)& $x_j^1=  x_k^1, j< k, x_j^2 > x_k^2$ &  $x_j^2:=0$ & Insert a leaf before the first child of $x_j^1$\\
(vii)& $x_j^1=  x_k^1, x_j^2 = x_k^2, x_j^3 > x_j^2$ &  $x_j^2:=0$ & Insert two leaves before the first child of $x_j^1$\\
(viii)& $x_j^2=0$ &  $x_j^3:=0$ & Insert a leaf before the first child of $x_j^1$\\
(ix)& $x_j^2 \neq 0, x_j^3 = 0$&  $x_j^2:=x_j^2+1$ & Insert a leaf after the $x_j^2$-th child of $x_j^1$\\\hline
\end{tabular}
\label{tab:reset}
\end{table*}
An illustration of a TI$_d$-generative ReLU is given in Fig.~\ref{fig:in_N}, 
where $m = 5$, $d = 4$ and $x=$1, 0, 3, 0, 2, 4, 1, 1, 3, 2, 5, 1, 4, 1, 3, 5. 
For convenience, we perform the insertion operations in the ascending order of the values $x^1$, i.e., in this case, the insertion is performed by considering the sequence 
$0, 0, 1, 3, 4, 1, 2, 1, 2, 1, 3, 5, 1, 5, 4, 3$. 
We discuss the insertion process as follows. 
For the node $x_2^1 = 0$, the bounds are $x_2^2 = 4$ and $x_2^3 = 2$, which are invalid as $D(x_2^1) = 1$. Therefore by applying refinements (i) and (ii), we set 
$x_2^2 := 0$ and $x_2^3 := 0$, and thus insert a leaf with label $1$ and DFS index $1$  (see Fig.~\ref{fig:in_N}), inward and outward edges with labels 1 and 6 at 1st and 2nd positions of the resultant Euler string $E(U)$, respectively. 
The bounds for the node $x_4^1 = 0$ are valid, and hence a new node with label $5$ is inserted with index $2$, as shown in Fig.~\ref{fig:in_N}, and insert inward and outward edges 5 and 10 at 3rd and 18th positions of $E(U)$. 
For the node $x_1^1 = 1$, the given lower bound and upper bound for the children are $x_1^2 = 2$ and $x_1^3 = 3$, which are valid as the number $D(x_1^1)$ of children of $x_1^1$ are 3 as depicted in Fig.~\ref{fig:in_N}. 
The 2nd and 3rd children of  $x_1^1$ have the DFS indices $3$ and $5$, respectively. 
Thus a new node ${x'}_1^1$ with label $4$ and index $5$ is inserted by setting the 2nd and 3rd children of  $x_1^1$ as the children of ${x'}_1^1$. The revised indices of the children are $6$ and $9$, resp., as shown in Fig.~\ref{fig:in_N}. 
Inward and outward edges with labels 4 and 9 are inserted at 7th and 16th position of $E(U)$, respectively.  
For the node $x_3^1 = 3$, the bounds are $x_3^2 = 1$ and $x_3^3 = 5$, where the upper bound is invalid as $D(x_3^1) = 1$. 
By applying the refinement (iii), we set $x_3^3 := 0$, and by 
(ix) we have $x_3^2 := 2$, therefore insert a leaf after the first child of  $x_3^1$ with DFS index $8$, and insert inward and outward edges with labels 3 and 8 at the 11th and 12th positions of $E(U)$, respectively.   
The resultant tree $U$ has the Euler string 
$E(U)=1, 6, 5, 3, 2, 7, 4$, $2, 4, 9, 3, 8, 7, 4, 9, 9, 8, 10$. 

 \comblue{The existence of TI$_d$-generative ReLU network is discussed in Theorem~\ref{thm:EInn}.}
\begin{theorem}
\label{thm:EInn}
For a rooted ordered tree $T$ of size $n+1$ with nodes from 
$\Sigma = \{1, 2, \ldots, m\}$, and a non-negative integer $d$, 
there exists a TI$_d$-generative ReLU network with size 
$\mathcal{O}(n^3)$ and constant depth.
\end{theorem}

A proof of Theorem~\ref{thm:EInn} and an explanation of each variable used in it are given in Example~\ref{exa:ins} in Appendix~\ref{sec:app}.
\section{TE$_d$-generative ReLU}\label{sec:EIGR}
Let $T$ be a tree with $n+1$ nodes and labels from 
$\Sigma = \{1, 2, \ldots, m\}$, Euler string 
$E(T)$, and a non-negative integer $d$. 
We define the {\em TE$_d$-generative ReLU} to be a 
ReLU neural network such that each Euler string over $\Sigma$ with edit distance at most $2d$ from $E(T)$ due to deletion, substitution and insertion operations can be obtained by appropriately choosing an input $ x = x_1, x_2, \ldots, x_{7d}$ of $7d$ nodes with $x_j \in [0,1)$, where $x_j$ is of the form $i\cdot \Delta$, $i$ is an integer and $\Delta$ is a small constant to set the number of digits to be considered after decimal in $x_j$. For example, if $\Delta = 0.01$, then $x_j$ can be any number in [0,1) with two decimal places, i.e., $x_j$ cannot take the value 0.011. . 
The input $x_1, x_2, \ldots, x_{d}$ (resp., $x_{d+1}, x_{d+2}, \ldots, x_{3d}$ and 
$x_{3d+1}, x_{3d+2}, \ldots, x_{7d}$) represents the nodes for deletion (resp., substitution and insertion) operations.
As a preprocessing step, the random inputs 
$x_j$ for $1 \leq j \leq 2d$ and $3d+1 \leq j \leq 6d$ (resp., 
$2d+1 \leq j \leq 3d$ and $6d+1 \leq j \leq 7d$) are converted into integers $i \in \{0, \ldots, n\}$
(resp.,  $\ell \in \Sigma$) if 
$x_j \in ((i-1)/n, i/n]$ (resp., $x_{\ell} \in [(\ell -1)/m, \ell/m]$ for $\ell = 1$, 
$((\ell -1)/m, \ell/m]$ otherwise).
For example, when $n = 5$ and $m  = 10$, the conversion table is given in Table~\ref{tab:con} in Appendix~\ref{sec:app}. 
To output a fixed number of nodes, we assume that $E(T)$ is padded with $2d$ $B$s, where $B \gg \max(m,n)$.
The network outputs  $2n+2d$ nodes $y = y_1, \ldots, y_{2n+2d}$ from which 
the desired string $E(U)$ can be obtained by trimming all $B$s from the start and end. 
More precisely, if $d_1$ (resp., $d_2$) deletion (resp., insertion) operations are performed, then $2d-2d_1$ (resp., ${2d_2}$) number of $B$s will be trimmed from the end (resp., start) of the output $y$ as shown in Fig.~\ref{fig:Uni_N} in Appendix~\ref{sec:app}. 
\begin{figure}[h!]
	\centering
	\includegraphics[scale = 0.45]{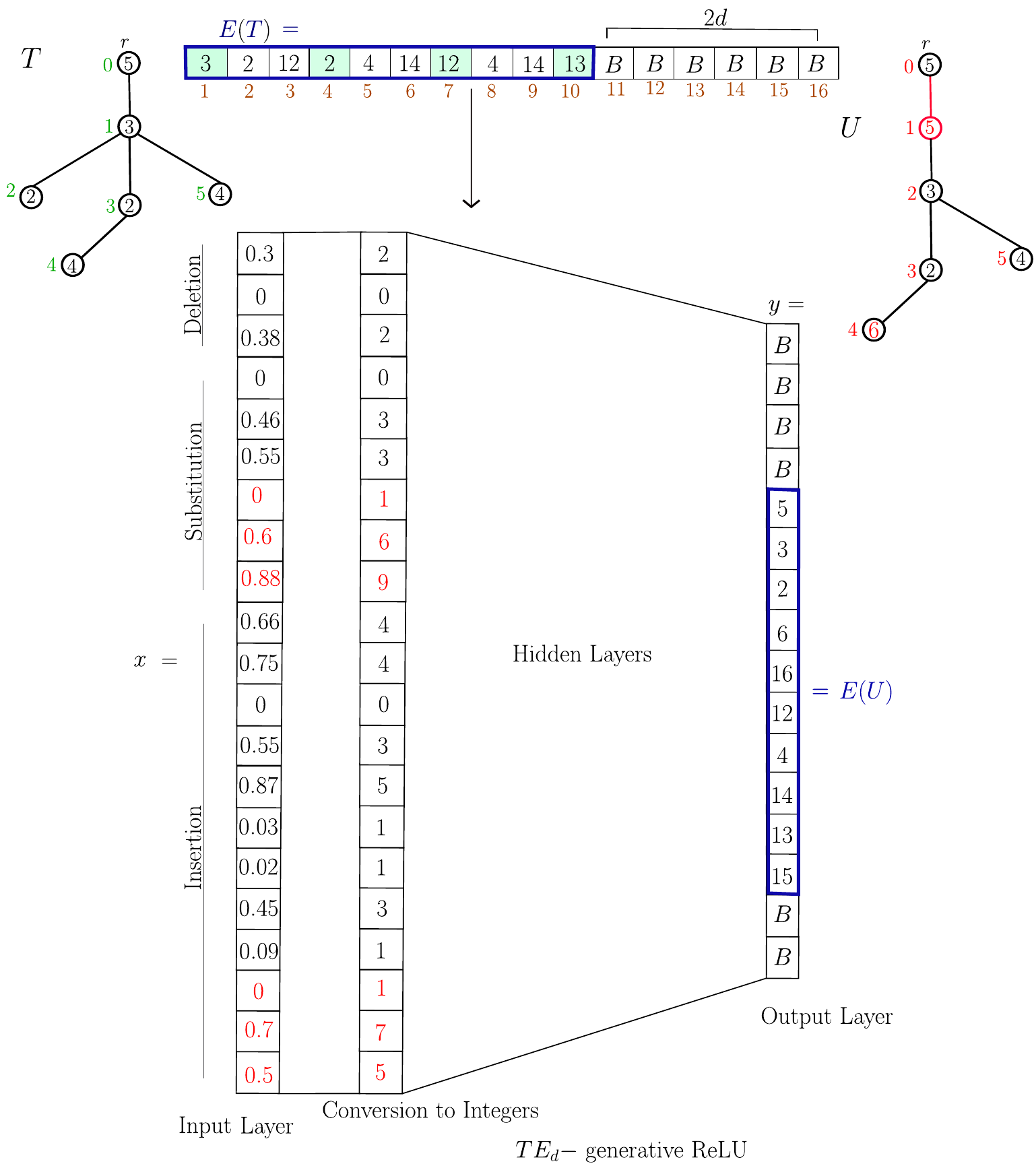}
	\caption{An illustration of a TE$_d$-generative ReLU for $d = 3$, 
	with the input layer 
$x=$0.3, 0, 0.38, 0, 0.46, 0.55, 0, 0.6, 0.88, 0.66, 0.75, 0, 0.55, 0.87, 0.03, 0.02, $0.45, 0.09, 0, 0.7, 0.5$, 
the integer conversion of $x$ in 2, 0, 2, 0, 3, 3, 1, 6, 9, 4, 4, 0, 3, 5, 1, 1, 3, 1, 1, 7, 5, padded Euler string $3, 2, 12, 2, 4, 14, 12, 4,$ $14, 13, B, B, B, B, B, B$ and the output layer $B, B, B, B, 5, 3, 2, 6, 16, 12, 4, 14, 13, 15, B, B, B, B$ with 
$d_1= d_2 = 1$. 
By trimming $B$s, we can get the resultant string $E(U)=5, 3, 2, 6, 16, 12, 4, 14, 13, 15$ obtained by deleting, substituting and inserting the indicated nodes. 
	}\label{fig:Uni_N}
\end{figure}

The existence of TE$_d$-generative ReLU network is discussed in Theorem~\ref{thm:Enn}.

\begin{theorem}
\label{thm:Enn}
For a rooted ordered tree $T$ of size $n+1$ with nodes from 
$\Sigma = \{1, 2, \ldots, m\}$, and a non-negative integer $d$, 
there exists a TE$_d$-generative ReLU network with size 
$\mathcal{O}(n^3)$ and constant depth.
\end{theorem}
A proof of Theorem~\ref{thm:Enn}, and an explanation of each variable used in Theorem~\ref{thm:Enn}  is given in Example~\ref{exa:Uni} in Appendix~\ref{sec:app}.
\pagebreak
\section{Computational Experiments}\label{sec:comp}
We implemented the proposed networks on a machine  with an AMD Ryzen 7 4800H,  Radeon Graphics processor (2.90 GHz), 16 GB of RAM, and Windows 11 Pro using Python {\tt version 3.11.6}.
Note that the proposed $\mathrm{TE}_d$-generative ReLU network supports all three edit operations, deletion, substitution, and insertion, whereas the $\mathrm{TI}_d$-generative ReLU network supports only insertion operations; however, $\mathrm{TI}_d$ has the same order $\mathcal{O}(n^3)$ as the corresponding $\mathrm{TE}_d$ Therefore, for analysis and comparison, we generated trees using the $\mathrm{TI}_d$ and $\mathrm{TE}_d$ networks for four trees $T_i$, $i = 1,2,3,4$, shown in Fig.~\ref{fig:ER}, with $8, 10, 6,$ and $11$ nodes, respectively, labels from $\Sigma=\{1,2,\ldots,10\}$, and distance bounds $d = 2, 2, 3,$ and $2$.
%
%
\begin{figure*}[h!]
	\centering
	\includegraphics[scale = 0.5,  trim = 11cm 0cm 12cm 0cm]{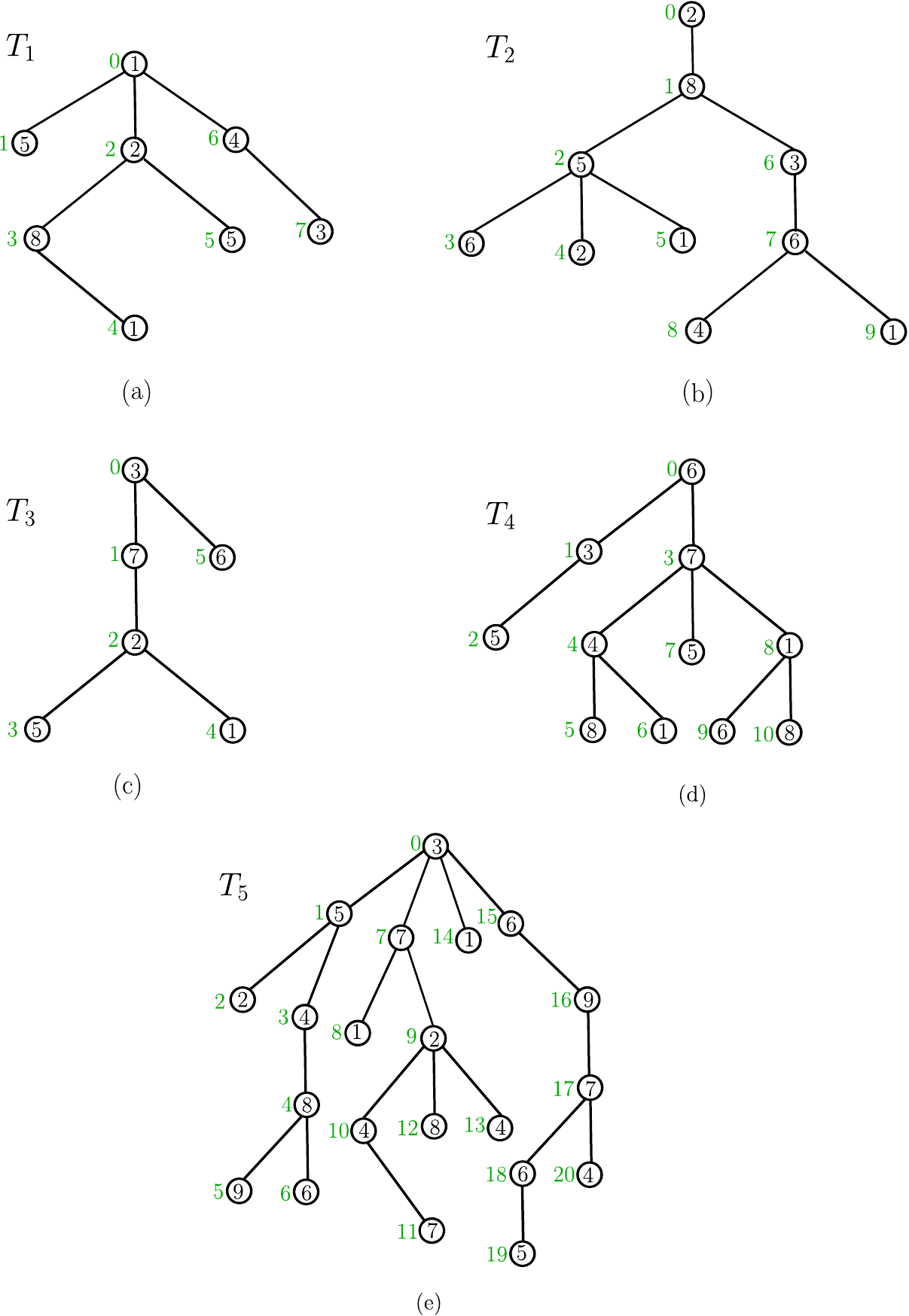}
	\caption{Trees used in the computational experiments of TI$_d$-generative ReLU and TE$_d$-generative ReLU.} \label{fig:ER}
\end{figure*}
For each tree $T_i$, $i = 1, 2,3, 4$ illustrated in Figs.~\ref{fig:ER}(a)-(d), we supplied the proposed $\mathrm{TI}_d$-generative ReLU network with a collection of input sequences $x$ that can cover all possible insertion operations to generate all possible similar trees with distance exactly $d$ due to the insertion operations. 
Duplicate outputs were removed in a post-processing step, yielding the set of distinct trees $U$ whose tree edit distance from $T_i$ is exactly $d$.
In these experiments, two (resp., two; three; and two) new nodes were inserted with labels 7, 7 (resp.,  9, 9;  8, 8, 8; and 2, 2) in $T_1$ (resp., $T_2$; $T_3$; and $T_4$) to generate similar trees of edit distance exactly $2$ (resp., 2; 3; and 2)
from $T_1$ (resp., $T_2$; $T_3$; and $T_4$) using the proposed 
$\mathrm{TI}_d$ network.
The computational results such as the number of nodes in each layer of the constructed TI$_d$ network along with the number of all distinct trees obtained after removing duplication from the similar trees generated by the proposed network for each input tree are provided in Table~\ref{tab:res-ins}.  
A summary of these computational results is given below. For $T_i$, let 
$(\rm{L}, 
\rm{TN}, 
\rm{MinN},
\rm{AvgN},
 \rm{MaxN}, 
 RT)^{\rm I}$$_i$ denote 
 the sequence of 
 number of hidden layers in TI$_d$, 
 total number of hidden nodes in TI$_d$, 
 minimum number of hidden nodes in TI$_d$, 
 average number of hidden nodes in TI$_d$,  
 maximum number of hidden nodes in TI$_d$, and 
 {running time (sec.) to generate one similar tree  using TI$_d$-generative ReLU network for $T_i$}, respectively. 
From these experiments, we have 
 $(57, 32876, 6, 576.77, 11903, 17.65)^{\rm I}_1$, 
$(57, 57324, 6, 1005.68, 24841, 46.23)^{\rm I}_2$,
$(57, 20556, 9, 360.63, $ $4500, 7.11)^{\rm I}_3$,
$(57, 73064, 6, 1281.82, 33860, 69.84)^{\rm I}_4$. 
These values are illustrated in Fig.~\ref{fig:plot}(a) for each tree $T_i$, $i = 1,2, 3, 4$. 
Observe that the depth remains fixed across all trees, confirming Theorem~\ref{thm:EInn}. 
Moreover, the widest layers are significantly larger than the others; for example, in $T_4$ with 11 nodes, the widest layer contains 33860 nodes which is over 26.41 times the average layer size (1281.82). 
Additionally, the total number of nodes grows faster than the tree size, e.g., while the number of nodes (11) in $T_4$ are roughly 1.8 times the number of nodes (6) in $T_3$, the corresponding total node count (73064) in TI$_d$ for $T_4$ is 3.56 times the node count (20556) in TI$_d$ for $T_3$. 
As a result, the running time also increases faster than the tree size.
 
 \begin{figure*}[t!]
	\centering
	\includegraphics[scale = 0.43]{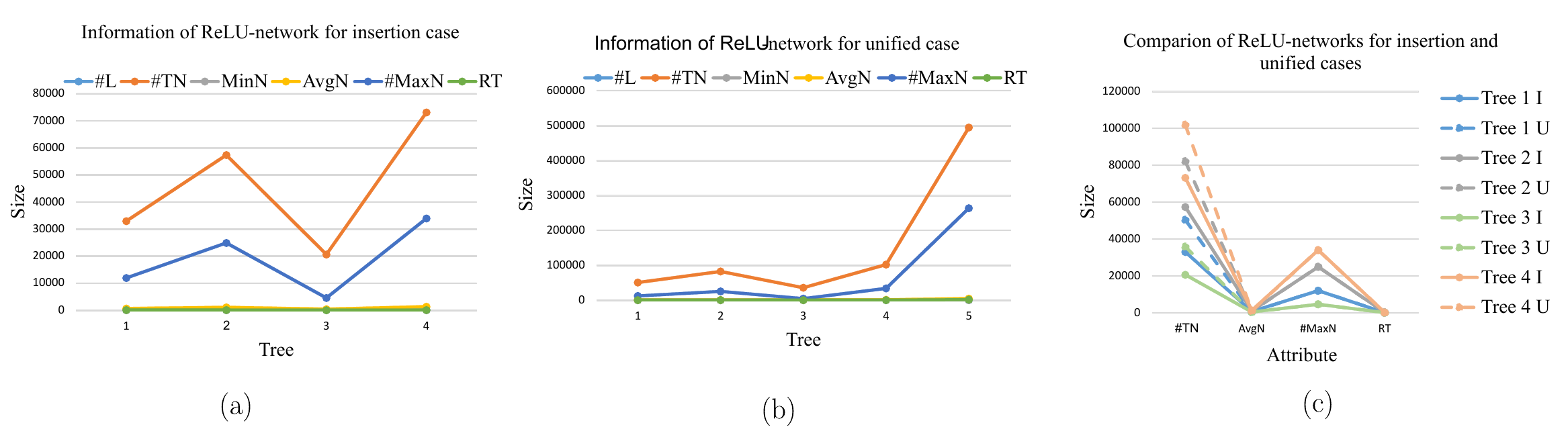}
	\caption{Information of the architectures of the proposed ReLU-networks for insertion and unified cases and their comparison: 
	(a) Information of $(\rm{L}, 
\rm{TN}, 
\rm{MinN},
\rm{AvgN},$ $
 \rm{MaxN}, 
 \rm{RT})^{\rm I}$$_i$ of the proposed ReLU-networks for insertion case for each tree $T_i$, $i = 1, 2, 3,4$; 
 (b) Information of $(\rm{L}, 
\rm{TN}, 
\rm{MinN},
\rm{AvgN},
 \rm{MaxN}, 
 {\rm RT})^{\rm E}$$_i$ of the proposed ReLU-networks for unified case for each tree $T_i$, $i = 1, 2, 3,4, 5$. 
The ranges of \rm{L}, \rm{MinN}, \rm{AvgN}, and {\rm RT} are much smaller than those of \rm{TN} and \rm{MaxN}, which causes their plots to appear compressed and overlap when shown on the same scale;  
 (c) A comparison of 
 the information of $( \rm{TN}, 
\rm{AvgN},
 \rm{MaxN}, \rm{RT})$$_i$ of the proposed ReLU-networks for insertion and unified cases for each tree $T_i$, $i = 1, 2, 3,4$. } \label{fig:plot}
\end{figure*}
Similarly, TE$_d$-generative ReLU neural networks were constructed for the trees $T_i$, $i = 1, 2, 3, 4,5$ given in Figs.~\ref{fig:ER}(a)-(e), where $T_5$ has $21$ nodes, labels from $\Sigma=\{1,2, \ldots,10\}$, and $d = 2$. 
The newly inserted (resp., substituted) nodes in $T_1$ are 0.7, 0.7 (resp., 0.55, 0.55), in $T_2$ are 0.9, 0.9 (resp., 0.7, 0.7), in $T_3$ are 0.8, 0.8, 0.8 (resp., 0.88, 0.88, 0.88), in $T_4$ are 0.2, 0.2 (resp., 0.9, 0.9) and in $T_5$ are 0.3, 0.3 (resp., 0.99, 0.99).

%

\begin{table}[H]
\centering
\caption{Experimental results for TI$_d$-generative ReLU}
\renewcommand{\arraystretch}{1.3} 
\setlength{\tabcolsep}{5pt} 
\footnotesize 

\begin{tabular}{|>{\centering\arraybackslash}p{2.2cm}|>{\centering\arraybackslash}p{11cm}|>{\centering\arraybackslash}p{2.2cm}|}  
\hline
Input trees & Size of each hidden layer of TI$_d$-generative ReLU & Number of all distinct similar trees obtained by TI$_d$ \\ \hline\hline
$T_1$, Fig.~\ref{fig:ER}(a)& 
\begin{minipage}{11cm}
\centering
\vspace{7pt}
36, 106, 806, 666, 232, 902, 400, 204, 23, 848, 218, 428, 218, 1073, 443, 758, 11903, 458, 670, 377, 252, 524, 520, 528, 524, 520,  
544, 528, 520, 540, 528, 518, 516, 528, 518, 518, 516, 640, 962, 482, 6, 26, 10, 22, 12, 6, 278, 23, 22, 352, 82, 90, 50, 26, 262, 78, 36
\vspace{3pt}
\end{minipage} & 318 \\ \hline
$T_2$, Fig.~\ref{fig:ER}(b)& 
\begin{minipage}{11cm}
\centering
\vspace{7pt}
44, 134, 1322, 1070, 368, 1446, 656, 332, 27, 1376, 350, 692, 350, 1737, 711, 1186, 24841, 730, 1074, 583, 392, 812, 808, 816, 812, 808,  
832, 816, 808, 828, 816, 806, 804, 816, 806, 806, 804, 960, 1522, 762, 6, 26, 10, 22, 12, 6, 342, 27, 26, 432, 102, 94, 54, 30, 330, 98, 44
\vspace{3pt}
\end{minipage} & 518 \\ \hline
$T_3$, Fig.~\ref{fig:ER}(c)& 
\begin{minipage}{11cm}
\centering
\vspace{7pt}
32, 82, 422, 362, 132, 529, 212, 112, 23, 452, 122, 232, 122, 573, 243, 474, 4500, 320, 378, 258, 164, 453, 447, 465, 456, 447,  
501, 465, 447, 489, 459, 441, 438, 456, 441, 441, 438, 576, 795, 399, 9, 51, 15, 45, 27, 9, 345, 23, 22, 436, 88, 172, 88, 28, 286, 82, 32
\vspace{3pt}
\end{minipage} & 546 \\ \hline
$T_4$, Fig.~\ref{fig:ER}(d)& 
\begin{minipage}{11cm}
\centering
\vspace{7pt}
48, 148, 1628, 1308, 448, 1766, 808, 408, 29, 1688, 428, 848, 428, 2129, 869, 1436, 33860, 890, 1312, 704, 474, 980, 976, 984, 980, 976,  
1000, 984, 976, 996, 984, 974, 972, 984, 974, 974, 972, 1144, 1850, 926, 6, 26, 10, 22, 12, 6, 374, 29, 28, 472, 112, 96, 56, 32, 364, 108, 48
\vspace{3pt}
\end{minipage} & 660 \\ \hline
\end{tabular}
\label{tab:res-ins}
\end{table}
The detailed computational results of these experiments are given in Table~\ref{tab:res-uni} which are summarized below. For $T_i$, let 
$(\rm{L}, 
\rm{TN}, 
\rm{MinN},
\rm{AvgN},
 \rm{MaxN}, 
 \rm{RT})^{\rm E}$$_i$ denote 
 the sequence of 
 number of hidden layers in TE$_d$, 
 total number of hidden nodes, in TE$_d$, 
 minimum number of hidden nodes in TE$_d$, 
 average number of hidden nodes in TE$_d$, 
 maximum number of hidden nodes in TE$_d$, and 
{running time (sec.) to generate one similar tree  using TE$_d$-generative ReLU network for $T_i$}, respectively. 
We have
$(108, 50360, 14, 466.30, 11917, 25.59)       ^{\rm E}_1$,
$(108, 82060, 14, 759.81$, $24859, 55.23)       ^{\rm E}_2$,
$(108, 35803, 19, $ $331.51, 4510, 11.36)        ^{\rm E}_3$,
$(108, 101930, 14$, $943.80, 33880, $ $85.04)       ^{\rm E}_4$,
$(108, 493790, 14, $ $4572.13, 263350, 1581.17)^{\rm E}_5$ 
for TE$_d$-generative ReLU networks.
These values are illustrated in Fig.~\ref{fig:plot}(b) for each tree $T_i$, $i = 1,2, 3, 4, 5$. 
Moreover a comparison of the TI$_d$ and TE$_d$ is given in Fig.~\ref{fig:plot}(c).
Observe that TE$_d$ networks are deeper and, on average, narrower than TI$_d$ networks for the same tree. Similarly, TE$_d$ networks exhibit larger minimum widths compared to TI$_d$ networks, while the maximum layer widths are nearly identical in both cases. These findings indicate that the insertion operation contributes the most to the overall size of the unified networks, which are composed of substitution, deletion, and insertion components. This observation also supports Theorems~\ref{thm:EInn} and~\ref{thm:Enn}.

\begin{table*}[t!]
\centering
\caption{Experimental results for TE$_d$-generative ReLU}
\renewcommand{\arraystretch}{1.3} 
\setlength{\tabcolsep}{5pt} 
\footnotesize 

\begin{tabular}{|>{\centering\arraybackslash}p{2.2cm}|>{\centering\arraybackslash}p{11cm}|>{\centering\arraybackslash}p{2.2cm}|}  
\hline
Input trees & Size of each hidden layer of TE$_d$-generative ReLU & Number of all distinct similar trees obtained by TE$_d$ \\ \hline\hline
$T_1$, Fig.~\ref{fig:ER}(a)& 
\begin{minipage}{11cm}
\centering
\vspace{7pt}
2016, 28, 120, 42, 34, 40, 42, 46, 30, 24, 22, 14, 66, 144, 1418, 482, 392, 1470, 696, 1650, 678, 354, 48, 30, 84, 30, 102, 30, 292, 82, 26, 70, 156, 54, 124, 908, 374, 262, 934, 458, 1816, 640, 444, 80, 164, 80, 276, 78, 36, 22, 50, 148, 50, 120, 890, 356, 218, 806, 414, 218, 37, 862, 232, 442, 232, 1087, 457, 772, 11917, 472, 684, 391, 266, 538, 534, 542, 538, 534, 558, 542, 534, 554, 542, 532, 542, 532, 532, 530, 654, 976, 496, 20, 40, 24, 36, 26, 20, 292, 37, 36, 366, 82, 90, 50, 26, 262, 78, 36
\vspace{3pt}
\end{minipage} & 747 \\ \hline
$T_2$, Fig.~\ref{fig:ER}(b)& 
\begin{minipage}{11cm}
\centering
\vspace{7pt}
2240, 28, 120, 42, 34, 40, 42, 46, 30, 24, 22, 14, 74, 172, 2082, 674, 564, 2146, 1024, 2454, 1002, 518, 56, 34, 100, 34, 122, 34, 372, 102, 30, 82, 196, 66, 156, 1452, 550, 406, 1486, 730, 2980, 1036, 712, 100, 208, 100, 352, 98, 44, 26, 62, 188, 62, 152, 1430, 528, 350, 1322, 674, 350, 45, 1394, 368, 710, 368, 1755, 729, 1204, 24859, 748, 1092, 601, 410, 830, 826, 834, 830, 826, 850, 834, 826, 846, 834, 824, 834, 824, 824, 822, 978, 1540, 780, 24, 44, 28, 40, 30, 24, 360, 45, 44, 450, 102, 94, 54, 30, 330, 98, 44
\vspace{3pt}
\end{minipage} & 1223 \\ \hline
$T_3$, Fig.~\ref{fig:ER}(c)& 
\begin{minipage}{11cm}
\centering
\vspace{7pt}
2688, 42, 204, 63, 51, 60, 63, 69, 45, 36, 33, 21, 89, 142, 1141, 405, 325, 1250, 562, 1314, 546, 290, 50, 34, 82, 34, 98, 34, 298, 88, 28, 84, 127, 48, 98, 498, 275, 165, 555, 265, 1355, 455, 355, 85, 175, 85, 295, 82, 32, 22, 42, 112, 42, 92, 482, 259, 122, 422, 222, 122, 33, 462, 132, 242, 132, 583, 253, 484, 4510, 330, 388, 268, 174, 463, 457, 475, 466, 457, 511, 475, 457, 499, 469, 451, 466, 451, 451, 448, 586, 805, 409, 19, 61, 25, 55, 37, 19, 355, 33, 32, 446, 88, 172, 88, 28, 286, 82, 32
\vspace{3pt}
\end{minipage} & 2525 \\ \hline
$T_4$, Fig.~\ref{fig:ER}(d)& 
\begin{minipage}{11cm}
\centering
\vspace{7pt}
 2352, 28, 120, 42, 34, 40, 42, 46, 30, 24, 22, 14, 78, 186, 2462, 782, 662, 2532, 1212, 2916, 1188, 612, 60, 36, 108, 36, 132, 36, 412, 112, 32, 88, 216, 72, 172, 1772, 650, 490, 1810, 890, 3670, 1270, 870, 110, 230, 110, 390, 108, 48, 28, 68, 208, 68, 168, 1748, 626, 428, 1628, 828, 428, 49, 1708, 448, 868, 448, 2149, 889, 1456, 33880, 910, 1332, 724, 494, 1000, 996, 1004, 1000, 996, 1020, 1004, 996, 1016, 1004, 994, 1004, 994, 994, 992, 1164, 1870, 946, 26, 46, 30, 42, 32, 26, 394, 49, 48, 492, 112, 96, 56, 32, 364, 108, 48
\vspace{3pt}
\end{minipage} & 1550 \\ \hline
$T_5$, Fig.~\ref{fig:ER}(e)& 
\begin{minipage}{11cm}
\centering
\vspace{7pt}
3472, 28, 120, 42, 34, 40, 42, 46, 30, 24, 22, 14, 118, 326, 8022, 2302, 2082, 8152, 3972, 9736, 3928, 1992, 100, 56, 188, 56, 232, 56, 812, 212, 52, 148, 416, 132, 332, 6732, 2090, 1770, 6810, 3370, 14530, 4930, 3330, 210, 450, 210, 770, 208, 88, 48, 128, 408, 128, 328, 6688, 2046, 1648, 6448, 3248, 1648, 89, 6608, 1688, 3328, 1688, 8289, 3369, 5296, 263350, 3410, 5052, 2614, 1774, 3580, 3576, 3584, 3580, 3576, 3600, 3584, 3576, 3596, 3584, 3574, 3584, 3574, 3574, 3572, 3904, 6930, 3486, 46, 66, 50, 62, 52, 46, 734, 89, 88, 912, 212, 116, 76, 52, 704, 208, 88
\vspace{3pt}
\end{minipage} & 6309 \\ \hline
\end{tabular}
\label{tab:res-uni}
\end{table*}
For each tree $T_i$, the inputs $x$ and the corresponding Euler strings $E(U)$ generated by the TI$_d$  and TE$_d$ networks are listed in the supplementary material S1 which is available on
~\url{https://github.com/MGANN-KU/TreeGen\_ReLUNetworks}.

\comblue{Additionally, we conducted experiments to generate trees with a given edit distance by using state-of-the-art graph generative models called GraphRNN by You et al.~\cite{You2018GraphRNN} and GraphGDP by  Huang et al.~\cite{Huang2022GraphGDP} for comparison. 
For this purpose, we randomly generated datasets of sizes 100, 150, 150, 200, and 800 of trees whose distances from $T_1$, $T_2$, $T_3$, $T_4$ and $T_5$ are $2, 2, 3, 2$ and $2$, respectively. 
For simplicity, the labels of the underlying trees were ignored.  
We trained GraphRNN (dependent Bernoulli variant) (resp., GraphGDP) on each of these datasets by using a 4-layer RNN (resp., 4-layer GNN) with hidden neuron size 128 (resp., 128).  Training ran for around 3000 epochs with batch size 32, learning rate 0.003 for GraphRNN and  0.00002 for GraphGDP by using 80\% of the input dataset, whereas 20\% dataset was used for testing, and default parameters were retained for other settings. 
As a result, for each tree, $1024$ (resp., 512) samples were generated by using each GraphRNN (resp., GraphGDP).
Computational results are given in Table~\ref{tab:grnn}. }

\comblue{The results demonstrate notable variability in the performance of both models across different tree structures. It is important to note that both GraphRNN and GraphGDP were expected to generate tree structures, yet neither model guarantees that all generated samples are valid trees. Additionally, among the generated trees, only those that satisfy the specified edit distance constraint are considered valid in this evaluation.}

\comblue{GraphRNN consistently generated a high percentage of trees (ranging from 86.3\% to 96.6\%), yet only a small fraction of these were valid with respect to the target distance constraint. 
For instance, for trees $T_1$, $T_2$, and $T_4$ with distance $d=2$, the proportion of valid trees remained low 6.4\%, 6.9\%, and 2.0\%, respectively. 
In the case of $T_5$, GraphRNN failed to produce any valid tree, despite generating over 92\% of trees. 
The only tree where GraphRNN achieved relatively high success was $T_3$ (with $d=3$), where 35.8\% of the generated samples were valid.}

\comblue{In contrast, GraphGDP generated fewer trees overall, with percentages ranging between 3.7\% and 48.6\% across different trees. However, it was able to generate a higher proportion of valid trees among those it produced, especially in the case of $T_3$, where 47.8\% of its outputs were valid. 
For the other trees, especially those with $d=2$, the valid tree percentages were noticeably lower: 9.2\% for $T_1$, 3.9\% for $T_2$, and 0.0\% for both $T_4$ and $T_5$.}
\begin{table}[h!]
\centering
\caption{Percentage of trees and valid trees generated by GraphRNN~\cite{You2018GraphRNN} and Huang et al.~\cite{Huang2022GraphGDP} with a given distance $d$}
\comblue{
\begin{tabular}{|c|c|c|c|c|c|c|}
\hline
\multirow{2}{*}{Tree} & \multirow{2}{*}{$n+1$} & \multirow{2}{*}{$d$} 
& \multicolumn{2}{c|}{GraphRNN~\cite{You2018GraphRNN}} 
& \multicolumn{2}{c|}{GraphGDP~\cite{Huang2022GraphGDP}} \\
\cline{4-7}
& & & Trees & Valid trees & Trees & Valid trees \\
\hline
$T_1$ & 8  & 2  & 86.3\% & 6.4\%  & 48.6\% & 9.2\%  \\
$T_2$ & 10 & 2  & 96.6\% & 6.9\%  & 35.0\% & 3.9\%  \\
$T_3$ & 6  & 3  & 92.0\% & 35.8\% & 47.9\% & 47.8\% \\
$T_4$ & 11 & 2  & 94.1\% & 2.0\%  & 32.8\% & 0.0\%  \\
$T_5$ & 21 & 2  & 92.1\% & 0.0\%  & 3.7\%  & 0.0\%  \\
\hline
\end{tabular}}
\label{tab:grnn}
\end{table}

\comblue{These findings suggest that GraphRNN and GraphGDP struggle to enforce the tree edit distance constraint, particularly as tree size and complexity increase, and may generate samples that are not even trees. 
}

\comblue{The datasets and the graphs generated by GraphRNN and GraphGDP are available in the supplementary materials S2 and S3, resp., which are available on~\url{https://github.com/MGANN-KU/TreeGen\_ReLUNetworks}}
\section{Conclusion} \label{sec:concl}
We study the existence of ReLU-based generative networks for producing trees similar to a given tree with respect to the tree edit distance.  
Our approach transforms a rooted, ordered, and vertex-labeled tree into a rooted, ordered, and edge-labeled directed tree.   
This directed tree is then encoded as an Euler string, which serves as both the input and output of the ReLU generative networks.
First, we proved that there exists a ReLU network of size $\mathcal{O}(dn)$ and constant depth that can identify the labels of $d$ inward edges in the Euler string.   
Furthermore, we showed that the outward edges corresponding to these $d$ inward edges can be identified using a ReLU network of size $\mathcal{O}(dn^2)$ and constant depth.   
Building on these results, we demonstrated that all similar trees generated through substitution (resp., deletion and insertion) operations can be constructed by ReLU networks of size   
$\mathcal{O}(dn^2)$ (resp., $\mathcal{O}(n^2)$ and $\mathcal{O}(n^3)$), all with constant depth.   
Finally, we proved that there exists a ReLU network of size $\mathcal{O}(n^3)$ and constant depth capable of generating any tree within distance  $d$ from the original tree under combined substitution, deletion, and insertion operations. These findings provide a theoretical foundation towards construction of  compact generative models and open new directions for efficient tree-structured data generation.

In this study, we do not consider scenarios where a newly inserted node becomes the parent of subsequent inserted nodes. This design choice simplifies the construction and supports tractable enumeration, but it limits the completeness of the editing model by excluding certain nested insertions. Addressing this limitation would extend the expressiveness of the framework and is a promising direction for future work.
 comparison with the state-of-the-art graph generative models GraphRNN by You et al.~\cite{You2018GraphRNN} and GraphGDP by 
Huang et al.~\cite{Huang2022GraphGDP} revealed that these models struggled to generate trees at a specified edit distance, particularly as tree size and structural complexity increased. For instance, GraphRNN and GraphGDP could not generate a single valid tree with 21 vertices and edit distance 2. While this experiment serves as an initial benchmark, it also underscores the need for further comparative evaluations with other models, such as TreeGAN and diffusion-based tree generators, to more clearly position the strengths of our proposed ReLU-based construction.
On the other side, our construction demonstrates that our proposed ReLU network with constant depth and polynomial size can generate all trees within a given edit distance; however, the number of neurons in certain hidden layers increases rapidly with tree size. For instance, generating trees with 21 nodes can require layers containing over 263350 neurons. Although this wide structure ensures theoretical expressivity and completeness, it presents challenges for scalability, implementation, and deployment in resource-constrained environments. \comb{To address this limitation, future work may focus on strategies such as width pruning, parameter sharing, and compression  to control and reduce the rapid growth in network width.} 
An implementation of the proposed networks is available at~\url{https://github.com/MGANN-KU/TreeGen\_ReLUNetworks}.

\section*{Author contributions}
Conceptualization, M.G. and T.A.; methodology, M.G. and T.A.; software, M.G.  validation, T.A.; formal analysis, M.G. and T.A.; investigation, M.G. data curation, M.G.; writing—original draft preparation, M.G. and T.A.; writing—review and editing, M.G. and T.A..; supervision, T.A.; project administration, T.A. All authors have read and agreed to the published version of the manuscript.

\section*{Acknowledgments }
The work of Tatsuya Akutsu was supported in part by Grants 22H00532 and 22K19830 from Japan Society for the Promotion of Science
(JSPS), Japan. 
The authors would like to thank Dr. Naveed Ahmed Azam, Quaid-i-Azam University Pakistan, for the useful technical discussions.

\section*{Conflict of interest}
All authors have no conflicts of interest in this paper.


\section{Appendix}\label{sec:app}
\subsection*{Proofs and Examples}
\begin{proof}[\proofname\ of Theorem 1]
Suppose $E(T)=t_1, t_2, \ldots, t_{2n}$ and $E(U) = u_1, u_2, \ldots, u_{2n}$ are two Euler strings over $\Sigma$ 
of trees $T$ and $U$, resp., such that 
$E(U)$ is obtained from $E(T)$ by substituting 
$x_{1+d}, x_{2+d}, \ldots, x_{2d}$ (resp., 
$x_{1+d}+m, x_{2+d}+m, \ldots, x_{2d}+m$) at the  inward edges (resp., outward edges) of  
$x_1, x_2, \ldots, x_{d}$.
We claim that the substitution operations on $E(T)$ to obtain $E(U)$ can be performed in the following three steps, where
$i \in \{1, 2, \ldots, 2n\}$, $j \in \{1, 2, \ldots, d\}$ and $C$ is a constant with $C\gg \max\{m, n\}$.
These steps are demonstrated on an example tree in~Example~\ref{exa:Sub}. \bigskip\\
Step 1. Remove non-zero repetitions from $x_1, \ldots, x_d$ by setting repeated non-zero values to $0$ to get $x'$.
\begin{align}
x'_{j} &= \max (x_{j}-C  \sum_{k=1}^{j-1} \delta(x_{j}, x_{k}) , 0). \label{eqx'1}
\end{align}
Step 2. Get the labels of the inward and outward edges that will remain unchanged after the substitution operation by using $P'_{i}$. 
The non-zero value of $P'_{i}$ is the unchanged label at the $i$-th entry in $E(T)$.  
\begin{align}
P_{ji} &= r'_{ji} + z'_{ji}, \label{eqP1}\\
P'_{i} &= t_i - \sum_{j=1}^{d} P_{ji} , \label{eqP'1}
\end{align}
where $r'_{ji}$ and $ z'_{ji}$ are the labels of the inward and outward edges corresponding to $x'$ which can be obtained by Lemmas~\ref{thm:inward} and ~\ref{thm:ebar4}, respectively.  \\
Step 3.  Perform substitution at the inward and outward edges corresponding to $x'$ by using $Q_{ji}$ and $Q'_{ji}$, respectively.
$R_{i}$ stores all the substituted labels in the resultant Euler string. 
\begin{align}
Q_{ji} &= \max (x_{j+d}-C\delta(r'_{ji}, 0) , 0),  \label{eqQ1}\\
Q'_{ji} &= \max (x_{j+d}+m -C\delta(z'_{ji}, 0) , 0),  \label{eqQ'1}\\
R_{i} &={ \sum_{j=1}^{d} (Q_{ji} + Q'_{ji})}. \label{eqR1}
\end{align}
Finally, combine the original and substituted entries to get the required Euler string $E(U)$ by 
\begin{align}
u_i &= P'_i + R_i . \label{equ1} 
\end{align}

All the above equations involve the maximum function or $\delta$ function which can be simulated 
by ReLU activation function by using \comblue{Proposition~1 by Ghafoor and Akutsu~\cite{MT2024}.}
Therefore there exists a TS$_d$-generative ReLU network with size 
$\mathcal{O}(dn^2)$ and constant depth.
\end{proof}
\begin{example}
\label{exa:Sub}
Consider the tree $T$ as shown in Fig.~\ref{fig:ET}(a) with 
$E(T)=3, 2, 7, 2, 4, 9, 7, 4$, $9, 8$, $d = 3$, $m = 5$, and 
$x=1, 3, 1, 5, 1, 2$. 
The resultant tree $E(U)$ obtained by applying the substitution operations on $E(T)$ due to 
the given $x$ is shown in Fig.~\ref{fig:sub_N}, where the 
repetition $x_3 = 1$ is ignored by setting it $0$. 
We demonstrate the process of obtaining 
$E(U)=5, 2, 7, 1$, $4, 9, 6, 4, 9, 10$ by using Theorem~\ref{thm:Esnn} as follows. 
\begin{figure*}[h]
	\centering
	\includegraphics[scale = 0.65]{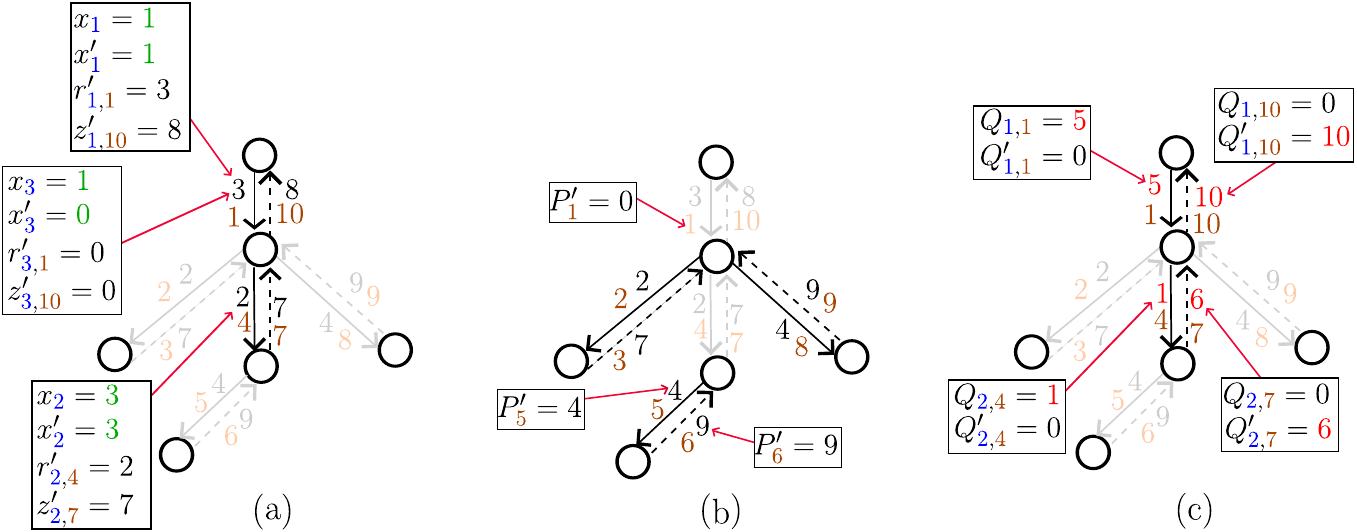}
	\caption{An illustration of the variables used in Theorem~\ref{thm:Esnn}.
	}\label{fig:sub}
\end{figure*}

\begin{longtable}{c p{16cm}}\addtocounter{table}{-1}
$x_j$ & Specify the inward edge and outward edge of $x_j \neq 0 $ to substitute $x_{j+d}$ and $x_{j+d}+m$. In this case $x=1, 3, 1, 5, 1, 2$ as illustrated in 
Fig.~\ref{fig:sub}(a), where the inward and outward edges that correspond to $x$ are depicted in black. \\
$x'_j$  & A variable that replaces repeated $x_j$ with zero, e.g.,  $x_1 =x_3 = 1$, therefore $x'_3 = 0$. The values of the variables are $x' = 1, 3, 0$ as illustrated in 
Fig.~\ref{fig:sub}(a).\\
$r'_{ji}$, $z'_{ji}$ & The labels of inward and outward edges of $x'_j$, resp., as explained in Examples~\ref{ex:in} and~\ref{ex:out}. 
The non-zero values of $r'_{ij}$ and $z'_{ji}$ are 
$r'_{1,1}=3$, $r'_{2,4}=2$, $z'_{1, 10}=8$ and $z'_{2,7}=7$, and are depicted in Fig.~\ref{fig:sub}(a). 
\\
 $P_{ji}$  & A variable that keeps the labels of the inward and outward edges simultaneously by taking the sum of $r'_{ji}$ and $z'_{ji}$. {Since $r'_{1,1}=3$ and $z'_{1,1}=0$ therefore, 
 $P_{1,1}=r'_{1,1} + z'_{1,1}=3$. Similarly, $P_{1, 10}=r'_{1,10} + z'_{1,10}=0+8=8$, $P_{2,4}=r'_{2,4} + z'_{2,4}=2+0=2$, $P_{2,7}=r'_{2,7} + z'_{2,7}=0+7=7$, and all other variables are zero}. \\
$P'_i$  & Stores the original entries of $E(T)$ where no substitution operation is performed by setting the $i$-th entry of $E(T)$ zero if $P_{ji}$ is non-zero for some $x'_j$, i.e., the inward and outward edges that correspond to $x'$ are set to zero in $E(T)$. 
For example, $P_{1,1}=3 \neq 0$, therefore $P'_{1} = 0$, whereas 
$P_{j,5} = 0$ for all $j$, therefore $P'_{5}= 4$ which is the $5$-th entry of $E(T)$. 
In this case $P' = [0, 2, 7, 0, 4, 9, 0, 4, 9, 0]$ as depicted in Fig.~\ref{fig:sub}(b).\\ 
$Q_{ji}$  & Performs substitution at the inward edges, i.e., 
 $Q_{ji} = x_{j+d}$ if $r'_{ji} \neq 0$, e.g., 
 $Q_{2,4}=1$ as $r'_{2,4}=2 \neq 0$, implying that the inward edge of $x'_2$ has index $4$ in $E(T)$ and is substituted by $1=x_{2+3}$. 
 Similarly  $Q_{1,1}=5$, and all other variables are zero as depicted in Fig.~\ref{fig:sub}(c).\\
$Q'_{ji}$  & Performs substitution at outward edges, i.e., 
$Q'_{ji} = x_{j+d} +m$ if $z'_{ji} \neq 0$, e.g., 
 $Q'_{2,7}=6$ as $z'_{2,7}=7 \neq 0$, implying that the outward edge of $x'_2$ has index $7$ in $E(T)$ and is substituted by $6=x_{2+3}+5$. 
 Similarly  $Q'_{1,10}=10$ , and all other variables are zero as depicted in Fig.~\ref{fig:sub}(c).\\
$R_{i}$  & Stores substituted value at index $i$. 
The values of the variables in this case are $R = [5, 0, 0, 1, 0, 0, 6, 0, 0, 10]$.\\
$u_{i}$  &The resultant string $E(U)$. 
The values of the variables are $u = [5, 2, 7, 1, 4, 9, 6, 4, 9, 10]$.
The corresponding tree $U$ is shown in Fig.~\ref{fig:sub_N}.
\end{longtable}

\end{example}

\begin{proof}[\proofname\ of Theorem 2]
Suppose $E(T)=t_1, t_2, \ldots, t_{2n}$ and $E(U) = u_1, u_2, \ldots, u_{2(n-d')}$, 
$d' \leq d$ are two Euler strings over $\Sigma$ 
corresponding to the trees $T$ and $U$,  resp., such that 
$E(U)$ is obtained from $E(T)$ by deleting at most $2d$ edges  of 
$x_1, x_2, \ldots, x_{d}$ from $T$.
We claim that the $E(U)$ can be obtained by using the following system of equations, where 
$i, \ell \in \{1, 2, \ldots, 2n+2d\}$, $j \in \{1, 2, \ldots, d\}$ unless stated otherwise, and $B, C$ are large numbers such that $C \gg B\gg \max(m,n)$ .\bigskip\\
Step 1. Remove non-zero repetitions from $x$ to get $x'$ as explained in Theorem~\ref{thm:Esnn}. \\
Step 2.
Identify the positions and labels of the inward and outward edges to be deleted by using Lemmas~\ref{thm:inward} and~\ref{thm:ebar4} as follows.
 
\begin{align}
q_i&= \sum_{j=1}^{d} q_{ji},   \label{eqq0}\\
r'_i&= t_i q_{i},  \label{eqr'0}\\
w'_{{\ell}i}&= 
\begin{cases}
0 &\text{~if~}  i\leq {\ell},\\
\max(\delta(s_i, r'_{\ell} + m) - \sum_{k=1, k \neq {\ell}}^{2n} w_{ki}, 0) &\text{~otherwise}, \label{eqw'ji}\\
\end{cases} \\
z'_i&= t_i \cdot \sum_{{\ell}=1}^{2n} w'_{{\ell}i}, \label{eqw0}\\
P_{i}&= r'_{i} + z'_i  \label{eqP0}.
\end{align}
%
Step 3.
Identify the labels to be retained to construct the resultant string after the deletion operations as follows. 
\begin{align} 
Q_i&= \delta(P_i, 0),  \label{eqQ0}\\
R_i  &= \max( B  \sum_{k=1}^{i} Q_k - C  \delta(Q_i, 0),  0),
 \label{eqR0}\\
{R'}_{i}^{j} &= \left[i  B \leq R_{i+j-1} \leq i  B + 1\right] \cdot t_{i+j-1} \text{~for~} 
i \in \{1, 2, \ldots, 2n\} \nonumber\\
&~~~~~~~ j \in \{1, 2, \ldots, 2d+1\}, \label{eqR'0}
\end{align}
Finally, get the required Euler string $E(U)$ from $y_i$ by removing $B$s. 
\begin{align}
y_i  &= \sum_{j=1}^{2d+1} {R'}^j_i \text{~for~} i \in \{1, 2, \ldots, 2n\}. \label{eqy0}
\end{align}

All the above equations involve the maximum function, $\delta$ function, or threshold function, which can be simulated 
by ReLU activation function by using \comblue{Proposition~1 by Ghafoor and Akutsu~\cite{MT2024}.} 
Therefore there exists a TD$_d$-generative ReLU network with size 
$\mathcal{O}(n^2)$ and constant depth.
\end{proof}
\begin{figure*}[ht!]
	\centering
	\includegraphics[scale = 0.65,  trim = 11cm 0cm 12cm 0cm]{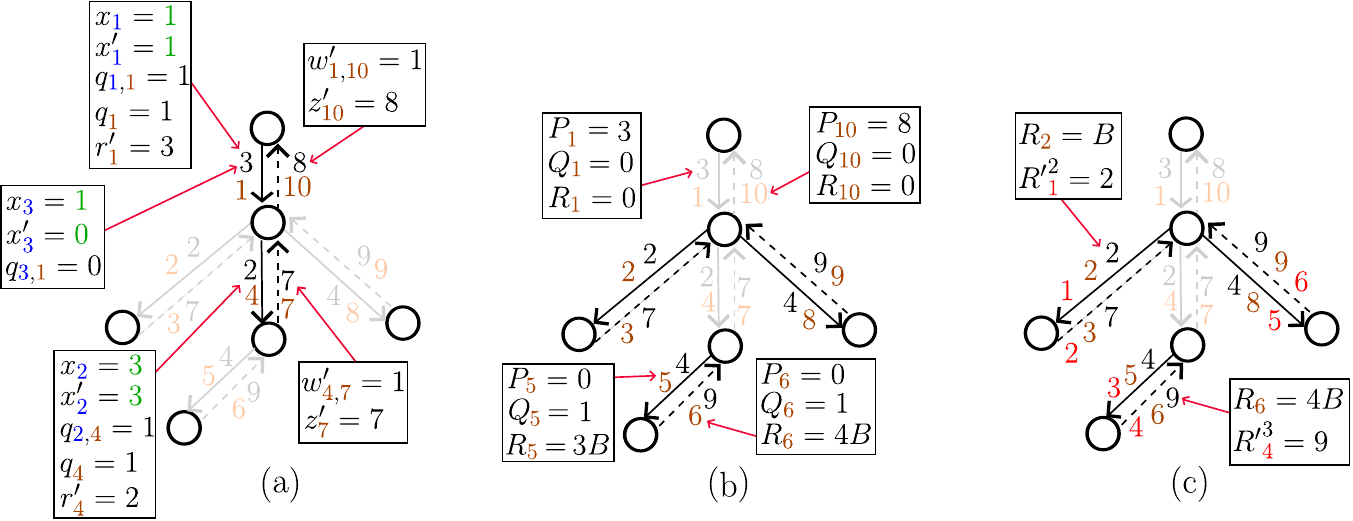}
	\caption{An illustration of the variables used in Theorem~\ref{thm:Ednn}.}\label{fig:Del}
\end{figure*}
\begin{example}
\label{exa:Del}
Reconsider the tree $T$ given in Fig.~\ref{fig:ET} with 
$E(T)=3, 2, 7, 2, 4, 9, 7, 4$, $9, 8$, $d = 3$, $m = 5$, and 
$x=1, 3, 1$. 
The resultant tree $U$ obtained by applying the deletion operations on $T$ due to 
the given $x$ is shown in Fig.~\ref{fig:Del_N}, where the 
repetition $x_3 = 1$ is ignored by setting it $0$ and deleting two $B$s from the padded  Euler string $E(T)$.  
We demonstrate the process of obtaining 
$E(U)=2, 7, 4, 9, 4, 9$ by using Eqs.~(\ref{eqq0})-~(\ref{eqy0}) as follows. 
An illustration of the variables used in these equations is given in Fig.~\ref{fig:Del}.

\begin{longtable}{c p{16cm}}\addtocounter{table}{-1}
$x_j$ & Specify the inward edge and outward edge of $x_j \neq 0 $ to be deleted. In this case $x=1, 3, 1$ as illustrated in 
Fig.~\ref{fig:Del}(a), where the inward and outward edges that correspond to $x$ are depicted in black. \\
$x'_j$  & A variable that replaces repeated non-zero $x_j$ with 0, e.g., 
$x_3 = 1$ is repeated, and therefore $x'_j = 0$. 
The values of the variables in this case are $x' = [1, 3, 0]$, and are depicted in 
Fig.~\ref{fig:Del}(a). \\
\multicolumn{2}{l}{$p_i, p_i', p_i'', q_{ji}$ and $s_i, w_{\ell i}$ are explained in Examples~\ref{ex:in} and~\ref{ex:out}, respectively.}\\
$q_{i}$  &A binary variable which is one if the inward edges of $x'_j \neq 0$ is 
the $i$-th entry of $E(T)$. 
In other words, this variable identifies the positions of the inward edges to be deleted from $E(T)$. 
In this case  $q'_{1}=q'_{4}=1$, since $q_{1,1} = q_{2,4} = 1$ as shown in Fig.~\ref{fig:Del}(a). \\
$r_i'$ & Stores the label of the inward edge of $x'_j \neq 0$ which is the $i$-th entry of $E(T)$. The non-zero values of this variable are $r'_1 = 3$ and $r'_4 = 2$ as shown in Fig.~\ref{fig:Del}(a). \\
$w'_{\ell i}$ &A binary variable to identify the outward edges of $x_j \neq 0$. More precisely, $w'_{\ell i}$ is one when $t_i$ is the outward edge of the inward edge $t_{\ell}$, and $t_{\ell}$ is the inward edge of $x_j$, e.g., $w'_{1,10}=w'_{4,7} = 1$  because $t_{10}$ and $t_7$ are the outward edges of the inward edges $t_{1}$ and $t_4$, resp., as depicted in Fig.~\ref{fig:Del}(a). All other values are zero.\\ 
$z_i'$ & Stores the label of the outward edge of $x'_j \neq 0$ which is the $i$-th entry of $E(T)$. The non-zero values of this variable are $z'_7 = 7$ and $z'_{10} = 8$ as shown in Fig.~\ref{fig:Del}(a). \\
$P_{i}$  & A variable that keeps the labels of both inward and outward edges to be deleted by taking the sum of $r'_i$ and $z'_i$. 
In this case, $P= [3, 0, 0, 2, 0, 0, 7, 0, 0, 8, 0, 0, 0, 0, 0, 0]$. \\
$Q_{i}$  &A binary variable to identify which entries of the padded $E(T)$ should be retained to get the string of  the  resultant tree, e.g., 
$P_5 = 0$ implies that the $5$th entry of $E(T)$ should appear in the resultant tree, and so $Q_5 = 1$ as depicted in Fig.~\ref{fig:Del}(b). 
Therefore $Q= [0, 1, 1, 0, 1, 1, 0, 1, 1, 0, 1, 1, 1, 1, 1, 1]$.\\
$R_{i}$  &Assigns weights to the retained entries in the ascending order, e.g., 
$t_5 = 4$ is the $3$rd entry to be retained as $Q_5 = 1$, and therefore $R_{5} = 3B$. In this case, 
$R = [0, B, 2B, 0, 3B, 4B, 0, 5B, 6B, 0, 7B, 8B, 9B, 10B, 11B, 12B]$ as depicted in Fig.~\ref{fig:Del}(b).\\
${R'}^j_i$  & Determines the label of the $i$-th entry of the resultant string obtained after the deletion operation. 
More precisely, the non-zero ${R'}^j_i$ is equal to the label of the entry  of $E(T)$ which has value $iB$
in $R$, e.g., when $i = 4$, $R_6 =  4B$ and the $6$th entry of $E(T)$ is $9$, therefore ${R'}^3_4 = 9$ where $i+j-1=4+3-1=6$, shows that the element of $6$th position of $E(T)$ has a shift of $j-1=2$ and becomes the $4$th element of resultant string as depicted in Fig.~\ref{fig:Del}(c).
The non-zero values of ${R'}^j_i$ are ${R'}^2_1 =2$, ${R'}^2_2=7$, ${R'}^3_3=4$, ${R'}^3_4=9$, ${R'}^4_5=4$, ${R'}^4_6=9$, ${R'}^5_7={R'}^5_8={R'}^6_9={R'}^6_{10}=B$.
\\
$y_{i}$  &Returns the non-zero entries of ${R'}^j_i$ for a fixed $i$ from which $E(U)$ can be obtained by removing $B$s. 
In this case $y = [2, 7, 4, 9, 4, 9, B, B, B, B]$ and so $ E(U) = 2, 7, 4, 9, 4, 9$ as required.
\end{longtable}

\end{example}
\begin{proof}[\proofname\ of Theorem 3]
Consider two trees $T$ and $U$ over $\Sigma$ such that 
$E(U)$ is obtained from $E(T)$ by inserting exactly $d$ inward and $d$ outward edges based on an appropriate $ x = x_1, \ldots, x_{4d}$. 
We claim that the insertion operations on $E(T)$ to obtain $E(U)$ can be performed in the following 10 steps, where $\ell \in \{0, 1, \ldots, 2n\}$, 
$i \in \{1, 2, \ldots, 2n\}$, $j \in \{1, 2, \ldots, d\}$, unless stated otherwise, and $C$ is a large number.\bigskip\\
Step 1. 
Determine the positions of inward and outward edges. 
The variable $q_j$ determines the position of the inward edge of $x^{1}_j$ by using 
Eq.~(\ref{eqq3}), and $b_{\ell}$ determines the position of each outward edge by using 
Eq.~(\ref{eqw3}), where $b_{0}$ corresponds to the root.  
\begin{align}
q_{j} &= \sum_{i=1}^{2n} i \cdot q_{ji}, \label{eqq'2}\\
b_{0} &= 2n+1,  b_{\ell} = \sum_{i=1}^{2n} i \cdot w_{{\ell}i},  
 \text{~for~}  {\ell} \in \{1, 2, \ldots, 2n\}. \label{eqz2}
\end{align}
Step 2. 
Determine the number of children of the node $x^{1}_j$, which is equal to the number of 
the descendant inward edges adjacent to the inward edge of $x^{1}_j$. 
The variable $A_{{\ell}i}$ is non-zero if and only if
the $i$-th edge (inward or outward) is adjacent to the $\ell$-th inward edge in $E(T)$, whereas $a_{\ell}$ is the number of adjacent inward edges and
$D_j$ is the number of children of $x^{1}_j$.
\begin{align}
A_{{\ell}i} &=  \max( \left[{\ell}+1 \leq i \leq b_{\ell} - 1\right] - \nonumber\\
&~~~~~~~  \sum_{k={\ell}+1}^{2n} \left[k+1 \leq i \leq b_k-1 \right], 0 ), 
  \label{eqA2}\\
a_{\ell} &= \sum_{i=1}^{2n} A_{{\ell}i} /2,  \label{eqa2}\\
D_{j} &= \sum_{k=1}^{n} \sum_{{\ell}=0}^{2n}  k \cdot (\delta(k, a_{\ell}) \wedge \delta({\ell}, q_j)). \label{eqD2}
\end{align}
Step 3. 
Refine the invalid lower bound $x^{2}$ and the upper bound $x^{3}$ as follows, where 
the refinements (i)-(ix) are performed by Eqs.~(\ref{eqQ2})-(\ref{eqQ32}), respectively. 
\begin{align}
Q^{1}_{j} &= \max (x^{2}_j - C(1-H(D_j -x^{2}_j) ),0), \label{eqQ2}\\
P^{1}_{j} &= \max (x^{3}_j - C(1-H(D_j -x^{3}_j) ),0), \label{eqP2}\\
P^{2}_{j} &= \max \left( P^{1}_{j} - C \cdot H(Q^{1}_{j}- P^{1}_{j} -1),0 \right),  \label{eqP22}\\
P^{3}_{j} &= \max ( P^{2}_{j} - C \cdot \sum_{k=j+1}^{d}(\delta(q_j , q_k) \wedge \nonumber\\
&~~~~~~~ H(P^{2}_{j}-Q^{1}_k -1) ),0 ),  \label{eqP32}\\
P^{4}_{j} &= \max (P^{3}_{j} - C \cdot \sum_{k=1}^{j-1}(\delta(q_j , q_k) \wedge \delta(Q^{1}_j, P^{3}_k) ),0 ), \label{eqP42}
\end{align}
\begin{align}
Q^{2}_{j} &= \max ( Q^{1}_{j} - C \cdot \sum_{k=j+1}^{d}(\delta(q_j , q_k) \wedge \nonumber\\
&~~~~~~~ H(Q^{1}_j - Q^{1}_k -1) ),0 ),  \label{eqQ22}\\
P^{5}_{j} &= \max (P^{4}_{j} - C \cdot \sum_{k=1, k\neq j}^{d}  (\delta(q_j , q_k) \wedge \delta(Q^{2}_j, Q^{2}_k)\wedge \nonumber\\
&~~~~~~~ H(P^{4}_{j}-Q^{2}_{j} -1)) ,0 ), \label{eqP52}\\
P^{6}_{j} &= \max (P^{5}_{j} - C\cdot \delta(Q^{2}_{j}, 0) ,0), \label{eqP62}\\
Q^{3}_{j} &= \max \left( Q^{2}_{j} +1- C (1-\delta(P^{6}_{j} , 0) \wedge H(Q^{2}_{j} -1) ),0 \right) \nonumber\\
&~~~~ + \max \left( Q^{2}_{j} - C (\delta(P^{6}_{j} , 0) \wedge H(Q^{2}_{j} -1) ),0 \right). \label{eqQ32}
\end{align}
\\
Step 4.  
For the inward edge at the $\ell$-th position, find the position of the $k$-th adjacent inward edge (child) and its outward edge in $E(T)$, 
where $k \in \{1, \ldots, n\}$.  
The variables ${G}^{k}_{\ell}$ and ${G'}^{k}_{\ell}$ are equal to $i$ if and only if the inward edge and the outward edge, resp,. of the $k$-th child of the $\ell$-th edge has index $i$ in $E(T)$. ${G}^{k}_{\ell}$ and ${G'}^{k}_{\ell}$ are $0$ if $\ell$ corresponds to an outward edge. The variable ${G''}^{k}_{\ell}$ identifies the position for the insertion of an edge after the last child. ${H}^{0}_{\ell j}$ and ${H'}^{0}_{\ell j}$ are used to insert leaves before the first child of a node. 
\begin{align}
F_{{\ell} i} &= \max(\sum_{k=1}^{i} A_{{\ell}k} /2 -C\delta(A_{\ell i, 0}), 0), \label{eqE2}\\
G^{k}_{\ell} &= \sum_{i=1}^{2n} i \cdot \delta(F_{{\ell} i} + 1/2 , k),  \label{eqG2}\\
{G'}^{k}_{\ell} &= \sum_{i=1}^{2n} i \cdot \delta(F_{{\ell} i} , k),  \label{eqG'2}\\
{G''}^{k}_{\ell} &= G^{k}_{\ell}  + \max ( {G'}^{k-1}_{\ell}+1 - C (1-\delta(G^{k}_{\ell} , 0) ), 0),   \label{eqG''2}\\
{J}^{0}_{\ell j} &=q_{j}+1, 
{J}^{k}_{\ell j} = {G''}^{k}_{\ell},   \label{eqH2}\\
{J'}^{0}_{\ell j} &= q_j, 
{J'}^{k}_{\ell j} = {G'}^{k}_{\ell}. \label{eqH'2}
\end{align}
Step 5. 
Find the positions before which new inward and outward edges to be inserted by using the variables  $L_{j}$ and $L''_{j}$, respectively. 
\begin{align}
L_{j} &= \sum_{k=0}^{n} \sum_{{\ell}=0}^{2n} \max ( {J}^{k}_{\ell  j}  - C (1-\delta({\ell} , q_j) \wedge \delta(k , Q^{3}_{j}) ) , 0), \label{eqL2}\\
L'_{j} &= \sum_{k=0}^{n} \sum_{{\ell}=0}^{2n} \max ( {J'}^{k}_{\ell  j}  - C (1-\delta({\ell} , q_j) \wedge \delta(k , P^{6}_{j})) , 0),  \label{eqL'2}\\
{L''}_{j} &= \max(L_{j}-1 - C(1-H ( L_{j}-{L'}_j) ),0) + \nonumber\\
&~~~~~~~\max({L'}_{j} - C \cdot H ( L_{j}-{L'}_j) ,0) + 1.  \label{eqL''2}
\end{align}
Step 6. 
Arrange $L_{j}$ in the ascending order, and then adjust the corresponding entries  ${L''}_{j}$ and $x^{4}_j$ accordingly.   
\begin{align}
R_{j} &= \max (\sum_{k=1}^{d} H (L_{j}-L_{k} )- \sum_{k=j}^{d} \delta(L_{k}, L_{j}), 0), \label{eqR2}\\
R'_{j} &= \sum_{k=1}^{d}  \max (L_{k} -C(1-\delta(k, R_{j}+1) ), 0),  \label{eqR'2}\\
R''_{j} &= \sum_{k=1}^{d}  \max (L''_{k} -C(1-\delta(k, R_{j}+1) ), 0),  \label{eqR''2}\\
x'^{4}_{j} &= \sum_{k=1}^{d}  \max (x^{4}_{k} -C(1-\delta(k, R_{j}+1) ), 0).  \label{eqX'42}
\end{align}
Step 7. 
Determine the increment and new positions of the entries of $E(T)$ in $E(U)$ due to the insertions. 
\begin{align}
M_{i} &= \sum_{j=1}^{d} \left( \delta(R'_{j}, i) + \delta({R''}_{j}, i) \right),   \text{~for~}   i \in \{1, 2, \ldots, 2n+1\}, \label{eqM2}\\
M'_{i} &= i + \sum_{k=1}^{i} M_{i},   \label{eqM'2}\\
N^{k}_{i}&= \max ( t_i - C (1-\delta(M'_{i} , i+k-1) ), 0),\nonumber\\
&~~~~~~~ \text{~for~} k \in \{1, \ldots, 2d+1\},  \label{eqN2}\\
N'_{h} &= \sum_{h=i+k-1} N^{k}_{i}, \text{~for~} h \in \{1, \ldots, 2n+ 2d\}. \label{eqN'2}
\end{align}
Step 8. 
Determine the positions of the new inward and outward edges in $E(U)$. 
\begin{align}  
S_{j} &= R'_{j} + 2(j-1) - \sum_{k=1}^{j-1}  H(R''_{k}-{R'}_{j} ) +
\sum_{k=1}^{j-1} \delta({R''}_{k}, R'_{j}) ,  
\label{eqS2}\\
S'_{j} &= R''_{j} +2d - \sum_{k=j+1}^{d}  H(R'_{k}-R''_{j} ) -
\sum_{k=1}^{d}  H(R''_{k}-R''_{j} ) +\nonumber\\
&~~~~~~~
\sum_{k=1}^{j-1} \delta(R''_k, R''_j). \label{eqS12}
\end{align}
Step 9. Arrange $S_j$ and $S_j'$ in the ascending order, and then adjust the corresponding labels of the new edges. 
\begin{align}
V_{k} &= S_k, 1 \leq k \leq d, V_{k} = S'_{k-d}, d+1 \leq k \leq 2d, \label{eqV2}\\ 
V'_{k} &= {x'}^{4}_k, 1 \leq k \leq d, V'_{k} = {x'}^{4}_{k-d} +m, d+1 \leq k \leq 2d, \label{eqV'2}\\ 
W_{k} &= \max (\sum_{r=1}^{2d} H (V_{k}-V_{r} )- \sum_{r=k}^{2d} \delta(V_{r}, V_{k}), 0), \label{eqW2}\\
W'_{k} &= \sum_{r=1}^{2d}  \max (V_{r} -C(1-\delta(r, W_{k}+1) ), 0),  \label{eqW'2}\\
W''_{k} &= \sum_{r=1}^{2d}  \max (V'_{r} -C(1-\delta(r, W_{k}+1) ), 0). \label{eqW''2}
\end{align}
Step 10. 
Insert the new inward and outward edges. 
\begin{align}
Z^{k}_{i}&= \max ( W''_k - C (1-\delta(W'_{k} , i+k-1) ), 0),\nonumber\\
&~~~~~~~ \text{~for~} k \in \{1, \ldots, 2d\}, i \in \{1, \ldots, 2n+1\},  \label{eqZ2}\\
Z'_{h} &= \sum_{h=i+k-1} Z^{k}_{i}, \text{~for~} h \in \{1, 2, \ldots, 2(n+d)\}. \label{eqGZ'}
\end{align}
Finally obtain the Euler string of the desired tree as follows:
\begin{align}
u_h &= N'_h + Z'_h, \text{~for~} h \in \{1, 2, \ldots, 2(n+d)\}. \label{eqy2} 
\end{align}
%
Notice that all the equations involve maximum function, Heaviside function, $\delta$
 or $[a \geq \theta]$ function which can be simulated by ReLU activation function
 by using \comblue{Theorem~1 by Kumano and Akutsu~\cite{KA2022} and Proposition~1 by Ghafoor and Akutsu~\cite{MT2024}.} 
The number of variables in these equations is $\mathcal{O}(n^3)$. 
Therefore we can construct a  TI$_d$-generative ReLU with size 
$\mathcal{O}(n^{3})$ and constant depth.
\end{proof}
\begin{example}
\label{exa:ins}
Reconsider the rooted tree $T$ with 
$E(T)=3, 2, 7, 2, 4, 9, 7, 4, 9, 8$, as shown in Figure~\ref{fig:ET}(a),  
$d = 4$. 
We discussed, in detail, the process of insertion to obtain  $E(U)
=1, 6, 5, 3, 2, 7, 4$, $2, 4, 9, 3, 8, 7, 4, 9, 9, 8, 10$ by using 
$x=1, 0, 3, 0, 2, 4, 1, 1, 3, 2, 5, 1, 4, 1, 3, 5$ in Fig.~\ref{fig:in_N}, and demonstrate the same by using 
Eqs.~(\ref{eqq'2})-(\ref{eqy2}) as follows. \\

\begin{figure*}[t!]
	\centering
	\includegraphics[scale = 0.6,  trim = 11cm 0cm 12cm 0cm]{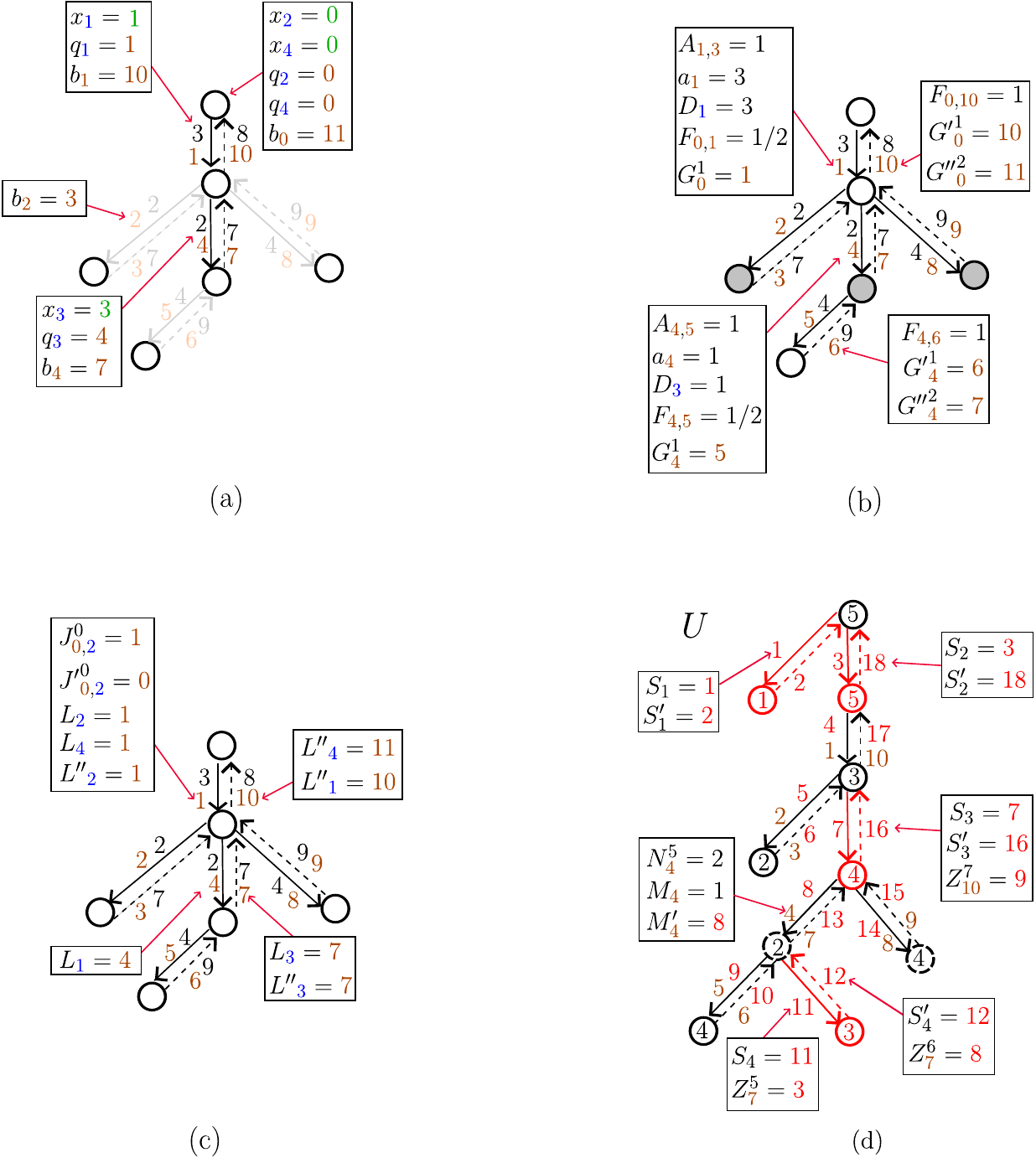}
	\caption{An illustration of the variables used in Theorem~\ref{thm:EInn}.}\label{fig:Ins}
\end{figure*}
Step 1. Determine the positions of inward and outward edges. 

\begin{longtable}{c p{16cm}}\addtocounter{table}{-1}
$x^1_j$ & Specify the node of $x^1_j$ for insertion with bounds $x^2_{j}$ and $x^3_{j}$ on the children and insertion value $x^4_{j}$.  
In this case $x=1, 0, 3, 0, 2, 4, 1, 1, 3, 2, 5, 1, 4, 1, 3, 5$ as depicted in 
Fig.~\ref{fig:in_N}. \\
%
%
\multicolumn{2}{l}{$q_{ji}, w_{\ell i}$ are explained in Examples~\ref{ex:in} and~\ref{ex:out}, respectively.} \\
$q_j$  & A variable that gives the position of the inward edge of $x^1_j$ in $E(T)$, where we consider $q_j = 0$ for $x^1_j = 0$ and so $q_2=q_4=0$.  
For example, the inward edges of $x^1_1=1$ and $x^1_3=3$ have positions $1$ and $4$, resp., in 
$E(T)$, and therefore $q_1=1$ and $q_3=4$ as depicted in Fig.~\ref{fig:Ins}(a).\\
$b_{\ell}$  &A variable that stores the position of the outward edge corresponding to the inward edge, if any, at the $\ell$-th position of the Euler string, e.g.,  
$b_{1}=10$ since the inward edge at the 1st position has the outward edge at the 10th position as shown in Fig.~\ref{fig:Ins}(a). 
Similarly, $b_{0}=11= 2n+1$ (by default),  $b_{2}=3$, $b_{4}=7$, $b_{5}=6$, $b_{8}=9$, and all other values are zero.
\end{longtable}
\noindent 
Step 2. Determine the number of children of $x^1_j$.
\begin{longtable}{c p{16cm}}\addtocounter{table}{-1}
$A_{\ell i}$ & A binary variable which is one if the $i$-th inward or outward edge is adjacent with 
the $\ell$-th inward edge in the directed $T$, e.g.,  $A_{1,3} =1$ as the outward edge at 3rd position in the directed $T$ is adjacent with the inward edge at the 1st position as shown in Fig.~\ref{fig:Ins}(b).
Similarly, $A_{0,1} =A_{0,10} = 1$ (by default), $A_{1,2} = A_{1,4} =A_{1,7} =A_{1,8} =A_{1,9} =A_{4,5} =A_{4,6} =1$.\\
$a_{\ell}$  & A variable that gives the number of descendant inward edges that are adjacent with the $\ell$-th inward edges, e.g., 
$a_1=3$ as there are three inward edges at the positions $2, 4, 8$ that are adjacent with the 
edge at the 1st position as shown in Fig.~\ref{fig:Ins}(b).
In this case, $a_0=1$, and $a_4=1$. All other values are zero. \\
$D_j$  &This variable gives the number of children of $x^1_j$, e.g., 
when $x^1_1=1$
$D_1=3$ as $a_{1} = 3$, and there is an inward edge at the 1st position of $E(T)$. 
Similarly for $x^1_2=x^1_4=0$ and $x^1_3=3$, we have $D_2=D_4= 1$ and $D_3=1$, respectively.
The children of $x^1_1=1$ are shown in gray in Fig.~\ref{fig:Ins}(b).
\end{longtable}
\noindent
Step 3. Refine the invalid lower and upper bounds. 
\begin{longtable}{c p{16cm}}\addtocounter{table}{-1}
$Q^1_{j}$ & A variable that sets $x^2_j:=0$ if $x^2_j>D_j$ following the refinement (i) of 
Table~\ref{tab:reset}. 
This means that the lower bound is set to $0$ if it is greater than the number of children. 
In this case, $Q^1_{2}=0$ because $x^2_2=4>D_2=1$. Whereas $Q^1_{1}=x^2_1=2$, $Q^1_{3}=x^2_3=1$ and $Q^1_{4}=x^2_4=1$.\\
$P^1_{j}$ & A variable that sets $x^3_j:=0$ if $x^3_j>D_j$ following the refinement (ii) of 
Table~\ref{tab:reset}. 
This means that the upper bound is set to $0$ if it is greater than the number of children. Here, $P^1_{2}=P^1_{3}=0$ because $x^3_2=2>D_2=1$ and $x^3_3=5>D_3=1$.  Whereas $P^1_{1}=x^3_1=3$ and $P^1_{4}=x^3_4=1$.\\
$P^2_{j}$ & A variable that sets the upper bound $P^1_{j}:=0$
following the refinement (iii) of 
Table~\ref{tab:reset}. 
This means that the upper bound is set to $0$ if it is smaller than the lower bound. Here, $P^2_{j}=P^1_{j}$. \\
$P^3_{j}$ & A variable that sets $P^2_{j}:=0$
following the refinement (iv) of 
Table~\ref{tab:reset}.
Here, $P^3_{j}=P^2_{j}$. \\
$P^4_{j}$ & A variable that sets $P^3_{j}:=0$ 
following the refinement (v) of 
Table~\ref{tab:reset}.
Here, $P^4_{j}=P^3_{j}$.\\
$Q^2_{j}$ & A variable that sets $Q^1_{j}:=0$ 
following the refinement (vi) of 
Table~\ref{tab:reset}.
Here, $Q^2_{j}=Q^1_{j}$. \\
$P^5_{j}$ & A variable that sets $P^4_{j}:=0$ {following the refinement (vii) of Table~\ref{tab:reset}. In $P^5_{j}$, $\delta(q_j , q_k)$ (resp., $\delta(Q^{2}_j, Q^{2}_k),  H(P^{4}_{j}-Q^{2}_{j} -1)$ ) correspond to $x_j^1= x_k^1$ (resp., $x_j^2 = x_k^2, x_j^3 > x_j^2$) of the refinement (vii). If these three are equal to 1 then $P_j^4$ which is refined form of $x_j^2$ becomes $0$}.  In this case, $P^5_{j}=P^4_{j}$. \\
$P^6_{j}$ & A variable that sets $P^5_{j}:=0$ if $Q^2_{j}=0$
following the refinement (viii) of 
Table~\ref{tab:reset}. 
This means that the upper bound is set to $0$ if the lower bound is 0. 
Here, $P^{6}_{j}=P^5_{j}$. \\
$Q^3_{j}$ & A variable to compute $Q^2_{j}+1$ if $P^6_{j}=0$ 
following the refinement (ix) of 
Table~\ref{tab:reset}.
Here, $Q^{3}_{3}=2$ as $P^6_{3}=0$, whereas $Q^{3}_{j}=Q^{2}_{j}$ for $j=1,2,4$.\\
\end{longtable}
\noindent
Step 4. Identify the position of the $k$-th child of a given node. 
\begin{longtable}{c p{16cm}}\addtocounter{table}{-1}
$F_{\ell i}$ & 
$F_{\ell i} = k$ (resp., $F_{\ell i} = k-{1/2}$) represents that the outward (resp.,  inward) edge of the $k$-th child of the ${\ell}$-th inward edge has index $i$,
e.g., $F_{0,10}=1$ and $F_{0,1} =1/2$, show that the outward edge and the inward edge of the first child of the root have positions $10$ and $1$, respectively, as Fig.~\ref{fig:Ins}(b).  
Similarly, $F_{1, 3} =1$, $F_{1, 2} =1/2$, $F_{1, 7} =2$, $F_{1, 4} =3/2$, $F_{1, 9} =3$, $F_{1, 8} =5/2$, $F_{4, 6}=1$, $F_{4, 5}=1/2$  and all other values of this variable are $0$. \\
${G}^k_{\ell}$  & It gives the position of the inward edge of the $k$-th child of the inward edge at the ${\ell}$-th position, e.g., 
$G_{0}^1=1$ and $G_{4}^1=5$ show that the inward edges of the $1$st child of the root and the inward edge at $4$ are $1$ and $5$, resp., as shown in Fig.~\ref{fig:Ins}(b).
Similarly, $G_{1}^1=2, G_{1}^2=4, G_{1}^3=8, G_{4}^1=5$.\\
${G'}^k_{\ell}$ & It gives the position of the outward edge of the $k$-th child of the inward edge with position ${\ell}$, e.g., 
${G'}_{0}^1=10$ and ${G'}_{4}^1=6$ show that the outward edges of the $1$st child of the root and inward edge at $4$ are $10$th and $6$th positions, resp., as shown in the Fig.~\ref{fig:Ins}(b).
Similarly ${G'}_{1}^1=3, {G'}_{1}^2=7, {G'}_{1}^3=9, {G'}_{4}^1=6$.\\
${G''}^k_{\ell}$ & This variable determines the position of the child, if any, to be inserted to the right of the children of the inward edge at the $\ell$-th position, when 
$k$ is equal to $D(\ell)  +1 $. 
When $k < D(\ell)$, this variable determines the position of the $k$-th child, whereas
${G''}^k_{\ell}$ will be ignored  when $k >D(\ell)  +1$, e.g., 
when $\ell = 0$, $D(0) = 1$, ${G''}^2_{0} = 11$ means that an inward edge as a child of the root on the right will be inserted at the 11th position as shown in Fig.~\ref{fig:Ins}(b), 
 ${G''}^1_{0} = 1$ is the position of the 1st child, and  
 ${G''}^3_{0} = 1$ will be ignored. 
Similarly, ${G''}^4_{1}=10$ and ${G''}^2_{4}=7$.\\
${J}^k_{\ell j}$ & For $k = 0$, this variable determines the position of the inward edge of the child, if any, to be inserted to the left of children of the inward edge at the $\ell$-th position
when $\ell$ is the position of the inward edge of $x_j^1$. 
For example, ${J}^0_{0, 2}=1$ means that an inward edge as a child of the root will be inserted before the 1st position as shown in Fig.~\ref{fig:Ins}(c). Also, ${J}^0_{1, 1}=2$, ${J}^0_{4, 3}=5$, and ${J}^0_{0, 4}=1$. For $k\geq1$ and any $j$,  ${J}^k_{\ell j} = {G''}^k_{\ell}$. 
For example, ${J}^4_{1, 2}={G''}^4_{1}=10$ .\\
${J'}^k_{\ell j}$ & For $k = 0$, this variable determines the position of the outward edge of the child, if any, to be inserted to the left of children of the inward edge at the $\ell$-th position
when $\ell$ is the position of the inward edge of $x_j^1$. 
For example, ${J'}^0_{0, 2}=0$ means that an outward edge as a child of the root will be inserted after the 0th position as shown in Fig.~\ref{fig:Ins}(c). Also, ${J'}^0_{1, 1}=1$, ${J'}^0_{4, 3}=4$, and ${J'}^0_{0, 4}=0$. For $k\geq1$ and any $j$,  ${J'}^k_{\ell j} = {G'}^k_{\ell}$. For example, ${J'}^1_{4, 2}={G'}_{4}^1=6$\\
\end{longtable}
\noindent
Step 5. Identify the insertion positions. 
\begin{longtable}{c p{16cm}}\addtocounter{table}{-1}
$L_{j}$ & This variable identifies the position for the new inward edge to be inserted corresponding to $x^1_j$, e.g., 
$L_1 = 4$ means that the first new inward edge will be inserted at the 3rd position (before the 4th entry of $E(T)$). 
Similarly, we get $L = [4, 1, 7, 1]$, as shown in the
 Fig.~\ref{fig:Ins}(c).\\
${L'}_{j}$  & This variable keeps the valid ${J'}^k_{\ell j}$, i.e., when 
the position of the inward edge of $x_j^1$ is $\ell$ and $k$ is equal to the refined, if necessary, upper bound corresponding to $x_j^3$. Thus $L' = [9, 0, 4, 10]$. \\
${L''}_{j}$ & 
This variable identifies the position of the new outward edge to be inserted corresponding to $x^1_j$. {If the index of inward edge is greater or equal to the index of outward edge, $\max(L_{j}-1 - C(1-H ( L_{j}-{L'}_j) ),0)= L_{j}-1$ and 
$\max({L'}_{j} - C \cdot H ( L_{j}-{L'}_j) ,0) + 1=1$. Therefore ${L''}_{j}= L_{j}$, which is required. 
If the index of inward edge is smaller than the index of outward edge, 
$\max(L_{j}-1 - C(1-H ( L_{j}-{L'}_j) ),0)= 0$ and 
$\max({L'}_{j} - C \cdot H ( L_{j}-{L'}_j) ,0) + 1={L'}_{j} +1$. Therefore ${L''}_{j}={L'}_{j} +1$ , which is required.}
e.g., 
$L''_1 = 10$ means that the new outward edge corresponding to $x^1_1$ will be inserted at the 9th position (before the 10th entry of $E(T)$). 
Similarly, we get $L'' = [10, 1,7, 11]$, as shown in the Fig.~\ref{fig:Ins}(c).
\end{longtable}
\noindent
Step 6. Arrange $L_{j}$ in the ascending order, and adjust  ${L''}_{j}$, $x^{4}_j$ accordingly.   
\begin{longtable}{c p{16cm}}\addtocounter{table}{-1}
$R_{j}$ & This variable identifies the position of $L_j$ in the arranged $L$, e.g., 
$R_1 = 2$ means that in the ascending order, $L_1$ will appear at the 2nd position. 
So $R=[2, 0, 3, 1]$. \\
${R'}_{j}$  & This variable arranges the value of $L_j$ in the ascending order w.r.t. $R_j$. In this case, $ R' = [1, 1, 4, 7]$. \\
${R''}_{j}$ & This variable arranges the value of $L''_j$ w.r.t. ${R}_{j}$. 
So, $ R'' = [1, 11, 10, 7]$. \\
${x'}^4_{j}$ & This variable arranges the value of ${x}^4_{j}$ w.r.t. ${R}_{j}$. 
So, $ x'^4 = [1, 5, 4, 3]$.
\end{longtable}
\noindent
Step 7. Determine the increment and new positions of the entries of $E(T)$ in $E(U)$ due to the insertions.
\begin{longtable}{c p{16cm}}\addtocounter{table}{-1}
$M_{i}$ & This variable identifies the number of insertions before the $i$-th position that corresponds to some  $x^k_j$, 
e.g., $M_{4}=1$ means that there will be one insertion before the 4th position of $E(T)$ as shown in Fig.~\ref{fig:Ins}(d).
The non-zero values are  $M_{1}=3$, $M_{7}=2$, $M_{11}=1$. \\
${M'}_{i}$  & Determines the new position of the $i$-th entry of $E(T)$ by summing up the increments before it, e.g., 
$M'_4 = 8$ since there are four increments $M_1 = 3$ and $M_4 = 1$, which implies that the 4th entry of $E(T)$ will be at the 8th entry of $E(U)$, as shown in the Fig.~\ref{fig:Ins}(d). 
In this case $ M' = [4, 5, 6, 8, 9, 10, 13, 14, 15, 17]$. \\
${N}^k_{i}$ & This variable links the $i$-th position with its increment and label, e.g., 
$N_4^5=2$ means that the 4th entry of $E(T)$ has an increment of $k-1=5-1=4$ and label $2$, as shown in the Fig.~\ref{fig:Ins}(d). 
Similarly, $N_1^4=3, N_2^4=2, N_3^4=7, N_5^5=4, N_6^5=9, N_7^7=7, N_8^7=4, N_9^7=9, N_{10}^{8}=8$. \\
${N'}_{h}$ & Finally, this variable lists the $i$-th entry of $E(T)$ with an increment $k$ at the position $h$ if $h = i+k-1$. 
In this case $N' = [0, 0, 0, 3, 2, 7, 0, 2, 4, 9, 0, 0, 7, 4, 9, 0, 8, 0]$.
\end{longtable}
\noindent
Step 8. Determine the positions of the new inward and outward edges in $E(U)$. 
\begin{longtable}{c p{16cm}}\addtocounter{table}{-1}
$S_{j}$ & It gives the position of the new inward edge corresponding to $R'_j$, 
e.g., $S_j = 1$ means that an inward edge corresponding to $R'_1 = 1$ will be inserted at the 1st position of $E(U)$. 
Thus $S = [1, 3, 7, 11]$, as shown in the Fig.~\ref{fig:Ins}(d).  \\
${S'}_{j}$ & It gives the position of the new outward edge corresponding to $R'_j$. 
In this case $S' = [2, 18, 16, 12]$, as shown in the Fig.~\ref{fig:Ins}(d).  
\end{longtable}
\noindent
Step 9. Arrange $S_j$ and $S_j'$ in the ascending order, and then adjust the corresponding labels of the new edges. 
\begin{longtable}{c p{16cm}}\addtocounter{table}{-1}
$V_{k}$ & This variable concatenates $S$ and $S'$, 
and so $ V = [1, 3, 7, 11, 2, 18, 16, 12]$. \\
${V'}_{k}$ & This variable lists the label of the new inward edge or outward edge corresponding to $V_k$.  In this case $ V' = [1, 5, 4, 3, 6, 10, 9, 8]$.\\
$W_{k}$ & This variable identifies the position of $V_k$ in the arranged $V$. 
In this case $ W = [0, 2, 3, 4, 1, 7, 6, 5]$. \\
${W'}_{k}$ &  This variable arranges $V_{k}$ w.r.t. $W_{k}$. 
In this case, $W' = [1, 2, 3, 7, 11, 12, 16, 18]$. \\
${W''}_{k}$ & This variable arranges the value of the label ${V'}_{k}$ w.r.t. $W_{k}$. So, $ W'' = [1, 6, 5, 4, 3, 8, 9, 10]$. 
\end{longtable}
\noindent
Step 10. Insert the new inward and outward edges. 
\begin{longtable}{c p{16cm}}\addtocounter{table}{-1}
${Z}^k_{i}$ & The variable $Z^{k}_{i} = W''_k$ if and only if $W''_k$ will be inserted as a new label at the $(i+k-1)$-th position in the resultant string, e.g., $Z_7^5=3$ means that the 5th insertion has label $3$ is performed before $7$th position of $E(T)$, and at $i+k-1=11$-th position in the resultant string as shown in the Fig.~\ref{fig:Ins}(d). Also, $Z_1^1=1, Z_1^2=6, Z_1^3=5, Z_4^4=4, Z_7^6=8, Z_{10}^7=9, Z_{11}^8=10$, and all other values are $0$.\\
${Z'}_{h}$ & This variable lists the label of the inward edge or outward edge 
$W''_k$ at the position 
$h$ if $h = i+k-1$. Thus, $ Z' = [1, 6, 5, 0, 0, 0, 4, 0, 0, 0, 3, 8, 0, 0, 0, 9, 0, 10]$. \\
$u_h$ & This variable sum up the entries of $N'_h$ and $Z'_h$ to obtained the resultant Euler string which, in this case, is $E(U) = [1, 6, 5, 3, 2, 7, 4, 2, 4, 9, 3, 8, 7, 4, 9, 9, 8, 10]$.
\end{longtable}
\end{example}
\begin{proof}[\proofname\ of Theorem 4]
Consider two trees $T$ and $U$ over $\Sigma$ such that 
$E(U)$ is obtained from $E(T)$ by performing $d$ edit operations based on an appropriate $ x = x_1, \ldots, x_{7d}$. 
The edit operations on $E(T)$ to obtain $E(U)$ can be performed in the following  steps, where 
$B$ and $C$ are large numbers with $C \gg B \gg \max(m,n)$.\bigskip\\
Step 1. 
Convert the input $x_j$ into integers. 
 The variables $P'_j $ and $Q'_j$ store the integer values corresponding to $x_j$, i.e.,  
 $P'_j = i$ where $i \in \{0, \ldots, n\}$,  (resp., $Q'_j = \ell$ where $\ell  \in \Sigma$) if and only if  $P_i^j = 1$ (resp., $Q_\ell^i = 1$).
\begin{align} 
P_{i}^{j} &= \left[ (i-1)/n \leq x_{j} \leq i/ n \right] -  \delta(x_j, (i-1)/n),  
 \nonumber\\
&~~~~ \text{~for~} i \in \{0, 1, \ldots, n\}, j \in \{1, \ldots, 2d, 3d+1, \ldots, 6d\}, \label{eqP5}\\
Q_{\ell}^{j} &= 
\begin{cases}
\left[ (\ell-1)/m \leq x_{j} \leq {\ell}/ m \right] ~~~~~~~~~\text{~if~}  {\ell} =1, \\[5pt]
\left[ (\ell-1)/m \leq x_{j} \leq {\ell}/ m \right] -   ~~~~~\text{~if~}  
{\ell} \in \{2, \ldots, m\},  \\
~~\delta(x_j, (\ell-1)/m)  , \label{eqQ6}
\end{cases}
 \nonumber\\
&~~~~ \text{~for~}  j \in \{2d+1, \ldots, 3d, 6d+1, \ldots, 7d\},\\
P'_{j} &= \sum^{n}_{i=0} p_{i}^{j} \cdot i  \text{~for~} 
 j \in \{1, \ldots, 2d, 3d+1, \ldots, 6d\}, \label{eqP'5}\\
Q'_{j} &=\sum^{m}_{{\ell}=1}  q_{\ell}^{j} \cdot {\ell}  \text{~for~}  j \in \{2d+1, \ldots, 3d, 6d+1, \ldots, 7d\}. \label{eqQ'5}
\end{align}
Step 2. 
Ignore $x_j$, $1 \leq j \leq 2d$, which are zero or repeated  to avoid redundant deletion and substitution operations. 
Similarly, ignore $x_j$, $1 \leq j \leq 2d$ which has index greater than $d$ among the non-zero and non-repeated positions. 
For  $x_{j=3d+h}$, $1 \leq h \leq d$ set weights $d-h+1$.
Finally, identify the valid operation positions in $x$ with  index at most $d$ using 
$R'_j$.  
\begin{align}
R_{j} &= 
\begin{cases}
\max (1- (\delta(P'_j, 0) + \sum^{j-1}_{k=1} \delta(P'_j, P'_k) ), 0)    \\~~~\text{~for~}  j \in \{1, \ldots, d\}, \\[10pt]
\max (1- (\delta(P'_j, 0) + \sum^{j-1}_{k=d+1} \delta(P'_j, P'_k) ), 0)    \\~~~\text{~for~}  j \in \{d+1, \ldots, 2d\},\\[10pt]
d+1- \sum^{j}_{k=3d+1}  H(P'_k) ~~~\text{~for~}  j \in \{3d+1, \ldots, 4d\}  \label{eqR5},
\end{cases}\\[2pt]
R'_{j} &=
\begin{cases}
\left[ \sum^{j}_{k=1} R_{k} \geq d+1 \right]  \\ ~~~\text{~for ~}  j \in \{1, \ldots, 2d\}, \\[15pt]
\left[ \sum^{2d}_{k=1} R_{k} +  R_{j} \geq d+1 \right]  \\ ~~~\text{~for~}  j \in \{3d+1, \ldots, 4d\}. \label{eqR'5}
\end{cases}
\end{align}
Step 3. 
Set the value of the redundant positions and their corresponding bounds and values $B$.
\begin{align}
S_{j} &= \max ( B - C (1 - R'_{j}) , 0) + \max (P_j' -\nonumber\\
&~~~~ C \cdot R'_{j}) , 0)    \text{~for~}  j \in \{1, \ldots, 2d, 3d+1, \ldots, 4d\}, \label{eqS5}\\
S'_{j} &= 
\begin{cases}
\max ( B - C (1 - \delta(S_{j-d}, B) ) , 0) + \max ( Q'_j -\\ ~~~ C \delta(S_{j-d}, B)  , 0) \text{~for~}  j \in \{4d+1, \ldots, 5d\}, \\
\max ( B - C (1 - \delta(S_{j-2d}, B) ) , 0) + \max ( Q'_j -\\ ~~~ C \delta(S_{j-2d}, B)  , 0) \text{~for~}  j \in \{5d+1, \ldots, 6d\},
\\
\max ( B - C (1 - \delta(S_{j-3d}, B)) , 0) + \max ( Q'_j - \\ ~~~C \delta(S_{j-3d}, B)  , 0) 
 \text{~for~}  j \in \{6d+1, \ldots, 7d\}. \label{eqS'5}
\end{cases}
\end{align}
{Finally, get the preprocessed input $x'_j$ as follows:}
\begin{align}
x'_{j} &= 
\begin{cases}
S_j ~~~\text{~for~}  j \in \{1, \ldots, 2d\}, \\
x_j ~~~\text{~for~}  j \in \{2d+1, \ldots, 3d\},\\
S_j ~~~\text{~for~}  j \in \{3d+1, \ldots, 4d\},\\
S'_j ~~~\text{~for~}  j \in \{4d+1, \ldots, 7d\}. \label{eqx'5}
\end{cases}
\end{align}
Step 4. 
Apply deletion operations on padded $E(T)$ by following Theorem~\ref{thm:Ednn} with 
$x'_{j}$, $j \in \{1, \ldots, d\}$ as an input to get $E(T')$. 
Apply substitution operations on $E(T')$ following Theorem~\ref{thm:Esnn} using $x'_{j}$, $j \in \{d+1, \ldots, 3d\}$, 
to get $E(T'')$. 
Apply insertion operations on $E(T'')$ using Theorem~\ref{thm:EInn} and $x'_{j}$, $j \in \{3d+1, \ldots, 7d\}$
to get $E(T''')$.
During substitution and insertion operations, replace Eq.~(\ref{eqr'3}) $r'_{ji} = t_i \cdot q_{ji}$ with $r'_{ji} = \max(t_i -C(1-\delta(q_{ji} , 1), 0)$. 
Similarly, replace Eq.~(\ref{eqr3}) $r_i = t_i \cdot p_i$ and Eq.~(\ref{eqz'3}) $z'_{ji}= t_i \cdot \sum_{{\ell}=1}^{2n} w'_{j{\ell}i}$  with 
$r_{i} = \max(t_i -C(1-\delta(p_{i} , 1), 0)$ and 
$z'_{ji}= \max(t_i -C(1-\delta(\sum_{{\ell}=1}^{2n} w'_{j{\ell}i} , 1), 0)$, respectively.
Finally, the required $E(U)$ can be obtained by trimming $B$s from $E(T''')$.
Notice that all equations involve maximum function, Heaviside function, $\delta$
 or $[a \geq \theta]$ function, and therefore due to Theorems~\ref{thm:Esnn},~\ref{thm:Ednn} and~\ref{thm:EInn}, there exists a 
 TE$_d$-generative ReLU network with size 
$\mathcal{O}(n^3)$ and constant depth.
\end{proof}
\begin{table}[h!]
\centering
\caption{Conversion table from real to integers.}
\footnotesize
\setlength{\tabcolsep}{2pt} %
\renewcommand{\arraystretch}{1.1} %

\begin{tabular}{|p{3.6cm}|p{3.6cm}|} %
\hline
\makecell[tc]{For positions $x_j$ \\  with $1 \leq j \leq 2d$ and ~\\$3d+1 \leq j \leq 6d$} & 
\makecell[tc]{For values $x_j$ \\ with $2d+1 \leq j \leq 3d$ and ~\\$6d+1 \leq j \leq 7d$} \\
\hline
$~~~~~~~~~~(-1/5,0] \rightarrow 0$ & $~~~~~~~~[0,1/10] \rightarrow 1$ \\
$~~~~~~~~~~(0,1/5] \rightarrow 1$  & $~~~~~~~~(1/10,2/10] \rightarrow 2$ \\
$~~~~~~~~~~(1/5,2/5] \rightarrow 2$ & $~~~~~~~~(2/10,3/10] \rightarrow 3$ \\
$~~~~~~~~~~(2/5,3/5] \rightarrow 3$ & $~~~~~~~~(3/10,4/10] \rightarrow 4$ \\
$~~~~~~~~~~(3/5,4/5] \rightarrow 4$ & $~~~~~~~~(4/10,5/10] \rightarrow 5$ \\
$~~~~~~~~~~(4/5,5/5] \rightarrow 5$   & $~~~~~~~~(5/10,6/10] \rightarrow 6$ \\
                          & $~~~~~~~~(6/10,7/10] \rightarrow 7$ \\
                          & $~~~~~~~~(7/10,8/10] \rightarrow 8$ \\
                          & $~~~~~~~~(8/10,9/10] \rightarrow 9$ \\
                          & $~~~~~~~~(9/10,10/10] \rightarrow 10$ \\
\hline
\end{tabular}
\label{tab:con}
\end{table}
\begin{example}
\label{exa:Uni} Reconsider the tree $T$ given in Fig.~\ref{fig:ET} with 
$E(T)=3, 2, 12, 2, 4, 14, 12, 4$, $14, 13$, $d = 3$, $m = 10$, and 
$x=0.3, 0, 0.38, 0, 0.46, 0.55, 0, 0.6, 0.88, 0.66, 0.75, 0, 0.55$, 
$0.87, 0.03, 0.02, 0.45, 0.09, 0, 0.7, 0.5$.   
We demonstrate the process of obtaining 
$E(U)$ by using Theorem~\ref{thm:Enn}. 
We explain the meaning of Eqs.~(\ref{eqP5})-~(\ref{eqx'5}), whereas the details of the  deletion, substitution and insertion operations can be followed from 
Examples~\ref{exa:Del},~\ref{exa:Sub} and~\ref{exa:ins}.

\begin{longtable}{c p{16cm}}\addtocounter{table}{-1}
${P}^j_i$ & Specify the interval for each position $x_j $, where 
$ j \in \{1, \ldots, 2d, 3d+1, \ldots, 6d\}$. 
${P}^j_i=1$ 
means that the $j$-th input lies in the $i$-th interval, i.e., the interval 
$( (i-1)/n , i/n]$. In this case ${P}^{2}_0={P}^{4}_0={P}^{12}_0$~$=
{P}^{15}_1={P}^{16}_1={P}^{18}_1={P}^{1}_2={P}^{3}_2=
{P}^{5}_3={P}^{6}_3={P}^{13}_3={P}^{17}_3=$~ ${P}^{10}_4=
{P}^{11}_4={P}^{14}_5=1$. All other values are zero. \\
${Q}^j_{\ell}$  & Specify the interval for each value $x_j$, where
 $j \in \{2d+1, \ldots, 3d, 6d+1, \ldots, 7d\}$.  
 ${Q}^j_{\ell}=1$ 
means that the $j$-th input lies in the ${\ell}$-th interval, i.e., 
for $\ell = 1$ (resp., $\ell \geq 2$), $x_j$ lies in $[0, \ell/m]$ (resp., $((\ell-1)/m, \ell/m]$). 
 In this case, ${Q}^{7}_1={Q}^{19}_1={Q}^{21}_5={Q}^{8}_6={Q}^{20}_7={Q}^{9}_9=1$, and all other values are zero.\\
$P'_{j}$  & Assigns each position $x_j$ an integer $i$ if $x_j$ belongs to the $i$-th interval, i.e., ${P}^j_i = 1$.
$P'_{1}=2$, $P'_{2}=0$, $P'_{3}=2$, $P'_{4}=0$, $P'_{5}=3$, 
$P'_{6}=3$, $P'_{10}=4$, $P'_{11}=4$, $P'_{12}=0$. \\
$Q'_{j}$  & Assigns each value $x_j$ an integer $\ell$ if $x_j$ belongs to the 
$\ell$-th interval, i.e., $Q_\ell^j = 1$. 
$Q'_{7}=1$, $Q'_{8}=6$, $Q'_{9}=9$, $Q'_{19}=1$, $Q'_{20}=7$, 
$Q'_{21}=5$.\\
$R_{j}$ &$R_{j}=1$ if $x_j$, $1 \leq j \leq 2d$, is a non-zero and non-repeated position.  
For $3d+1 \leq j \leq 4d$, $R_j$ is a weight from $\{d, d-1, \ldots, 1\}$ assigned to $x_j$ in the descending order. 
In this case $R_{1}=R_{5}=1$ for $1 \leq j \leq 2(3)$,
whereas for $3(3)+1 \leq j \leq 4(3)$, $R_{10}=3, R_{11}=2, R_{12}=1$.\\ 
$R'_{j}$ & The variable $R'_{j}=1$ if the position $x_j, 1 \leq j \leq 2d$ among the non-zero and non-repeated entries is at least $d+1$. 
Similarly, for $x_j, 3d+1 \leq j \leq 4d$, $R'_{j}=1$ if the sum of the number of non-zero and non-repeated position entries before $x_{3d+1}$ and weight $R_j$ is at least $d+1$. 
In this case, $R'_{10}=R'_{11}=1$. $R'_{j}=0$ for $x_j, 1 \leq j \leq 2d$.\\
%
$S_{j}$  & $S_j = B$ if $x_j$  has index at least $d+1$  among the positions, i.e., $S_{j}=B$ if $R'_{j}=1$ otherwise it is $P'_{j}$. In this case 
$S_{10}=S_{11}=B$. For all other values of $j$, $S_{j}=P'_{j}$\\
$S'_{j}$  & If $S_{j}=B$ for  $3d+1 \leq j \leq 4d$, then the corresponding bounds and values of $x_j$ are also set to $B$, i.e, $S'_{j} = B$. 
In this case 
$S'_{13}=S'_{14}=S'_{16}=S'_{17}=S'_{19}=S'_{20}=B$.\\
$x'_{j}$  &Gives the preprocessed input for edit operations; $x'=[2, 0, 2, 0, 3, 3, 1, 6, 9, B$, $B, 0, B, B, 1, B, B, 1, B, B, 5]$.
\end{longtable}

\noindent
Apply deletion operations on padded $E(T)$ to get $E(T')=3, 2, 4, 14, 12, 4, 14$, $13, B, B$, 
substitution operations to get  $E(T'')=3, 2, 6, 16, 12, 4, 14, 13, B, B$, and 
insertion operations to get  $E(T''')=B, B, B, B, 5, 3, 2, 6, 16, 12, 4, 14, 13, 15$, $B, B$. 
Finally, obtain $E(U)= 5, 3, 2, 6, 16, 12, 4, 14, 13, 15$ by trimming $B$s as shown in Fig.~\ref{fig:Uni_N}.  
\end{example}
\subsection*{Examples of Code Execution}
All codes are freely available at \url{https://github.com/MGANN-KU/TreeGen\_ReLUNetworks}. 
An explanation of the program codes is given below. \bigskip\\
The  file {\ttfamily Finding\_outward\_edges.py} contains an implementation of the proposed generative ReLU to find the indices and labels of outward edges of the given inward edges in the Euler string.\\
{\bf Input:} \\
		 {\ttfamily t} := Input Euler string of length {\ttfamily 2n}\\
		{\ttfamily m} := The size of the symbol set\\
		{\ttfamily x} := The string of length {\ttfamily d} to identify inward edges\\
{\bf Output:}\\  
		{\ttfamily y} := The Outward edges
	of {\ttfamily t} following {\ttfamily x}. \\
{\bf An example:} 
		{\ttfamily  t = 3, 2, 7, 2, 4, 9, 7, 4, 9, 8, \\
		d = 3,\\
		m = 5,\\
		x = 1, 3, 0, \\
		y = 10, 7, 0.}	\bigskip\\	
%
The  file {\ttfamily TS\_d.py}  contains an implementation of TS$_d$-generative ReLU  to generate Euler strings with a given Edit distance due to substitution. \\
{\bf Input:} \\
		 {\ttfamily t} := Input Euler string of length {\ttfamily 2n}\\
		{\ttfamily  d} := Edit distance\\
		{\ttfamily m} := The size of the symbol set\\
		{\ttfamily x} := The string of length {\ttfamily 2d} to identify substitution operation\\
{\bf Output:}\\  
		{\ttfamily u} := The Euler string obtained by applying substitution operation
	on {\ttfamily t} following {\ttfamily x} and has at most distance {\ttfamily 2d}. \\
{\bf An example:} 
		{\ttfamily  t = 3, 2, 7, 2, 4, 9, 7, 4, 9, 8,\\
		d = 3,\\
		m = 5,\\
		x = 1, 3, 1, 5, 1, 2,\\
		u = 5, 2, 7, 1, 4, 9, 6, 4, 9, 10.}		\bigskip\\
The  file {\ttfamily TD\_d.py}  contains an implementation of 
TD$_d$-generative ReLU  to
generate Euler strings with a given edit distance due to deletion 
operation only.  \\
{\bf Input:} \\
		 {\ttfamily t} := Input Euler string of length {\ttfamily 2n}\\
		{\ttfamily  d} := Edit distance\\
		{\ttfamily m} := The size of the symbol set\\
		{\ttfamily x} := The string of length {\ttfamily d} to identify deletion operation\\
{\bf Output:}\\  
		{\ttfamily u} := The Euler string obtained by applying deletion operation and trimming B 
	on {\ttfamily t} following {\ttfamily x} and has distance {\ttfamily 2d}. \\
{\bf An example:} 
		{\ttfamily  t = 3, 2, 7, 2, 4, 9, 7, 4, 9, 8,\\
		d = 3,\\
		m = 5,\\
		x = 1, 3, 1,\\
		u = 2, 7, 4, 9, 4, 9.}\bigskip\\
The  file {\ttfamily TI\_d.py}  contains an implementation of 
TI$_d$-generative ReLU  to
generate strings with a given edit distance due to insertion 
operation only.  \\
{\bf Input:} \\
		 {\ttfamily t} := Input Euler string of length {\ttfamily 2n}\\
		{\ttfamily  d} := Edit distance\\
		{\ttfamily m} := The size of the symbol set\\
		{\ttfamily x} := The string of length {\ttfamily 4d} to identify insertion operation\\
{\bf Output:}\\  
		{\ttfamily u} := The string obtained by applying insertion operation
	on {\ttfamily t} following {\ttfamily x} and has distance {\ttfamily 2d}. \\
{\bf An example:} 
		{\ttfamily  t = 3, 2, 7, 2, 4, 9, 7, 4, 9, 8,\\
		d = 4,\\
		m = 5,\\
		x = 1, 0, 3, 0, 2, 4, 1, 1, 3, 2, 5, 1, 4, 1, 3, 5,\\
		u = 1, 6, 5, 3, 2, 7, 4, 2, 4, 9, 3, 8, 7, 4, 9, 9, 8, 10.}\bigskip\\
The  file {\ttfamily TE\_d\_unified.py}  contains an implementation of  
TE$_d$-generative ReLU  to
generate strings with a given edit distance due to deletion,
substitution and insertion operations simultaneously. \\
{\bf Input:} \\
		 {\ttfamily t} := Input Euler string of length {\ttfamily 2n}\\
		{\ttfamily  d} := Edit distance\\
		{\ttfamily m} := The size of the symbol set\\
		{\ttfamily $\Delta$} := The small number\\
		{\ttfamily x} := The string of length {\ttfamily 7d} to identify substitution, insertion and deletion operations\\
{\bf Output:}\\  
		{\ttfamily u} := The string obtained by applying deletion, substitution, 
			 and insertion operations simultaneously
	on {\ttfamily t} following {\ttfamily x} and has at most distance {\ttfamily 2d}. \\
{\bf An example:} 
		{\ttfamily  t := 3, 2, 12, 2, 4, 14, 12, 4, 14, 13,\\}
		{\ttfamily d := 3,\\}
		{\ttfamily m := 10,\\}
		{\ttfamily $\Delta$ := 0.01 \\}
		{\ttfamily x := 0.3, 0, 0.38, 0, 0.46, 0.55, 0, 0.6, 0.88, 0.66, 0.75, 0, 0.55$, 
$0.87, 0.03, 0.02, 0.45, 0.09, 0, 0.7, 0.5,\\}
		{\ttfamily u := 5, 3, 2, 6, 16, 12, 4, 14, 13, 15.}	
\end{document}